\documentclass[11pt]{article}
\usepackage[T1]{fontenc}
\usepackage[margin=1in]{geometry}

\usepackage{microtype}
\usepackage{graphicx}
\usepackage{subfig}
\usepackage{caption}
\usepackage{booktabs} %

\usepackage{amsmath}

\usepackage[colorlinks=true,allcolors=blue,hyperfootnotes=false]{hyperref}
\usepackage[capitalise]{cleveref}

\title{Near-Optimal Data Source Selection for Bayesian Learning}

\author{%
Lintao Ye%
\thanks{Department of Electrical Engineering, University of Notre Dame; \texttt{lye2@nd.edu}.}
\and
Aritra Mitra%
\thanks{Department of Electrical and Systems Engineering, University of Pennsylvania; \texttt{amitra20@seas.upenn.edu}.}
\and
Shreyas Sundaram%
\thanks{School of Electrical and Computer Engineering, Purdue University; \texttt{sundara2@purdue.edu}.}
}

\RequirePackage{mathtools}
\RequirePackage{dsfont}
\RequirePackage{delimset}

\usepackage{amsthm}
\usepackage{amssymb}
\usepackage{thmtools}

\newtheorem{problem}{Problem}
\newtheorem{theorem}{Theorem}

\newtheorem{lemma}[theorem]{Lemma}
\newtheorem{remark}[theorem]{Remark}
\newtheorem{definition}[theorem]{Definition}

\newtheorem{assumption}{Assumption}

\usepackage{enumitem}
\usepackage[numbers]{natbib}
\usepackage{algorithm}
\usepackage[noend]{algpseudocode}

\algnewcommand{\IfThenElse}[3]{%
  \State \algorithmicif\ #1\ \algorithmicthen\ #2\ \algorithmicelse\ #3}

\begin{document}

\maketitle

\begin{abstract}
We study a fundamental problem in Bayesian learning, where the goal is to select a set of data sources with minimum cost while achieving a certain learning performance based on the data streams provided by the selected data sources. First, we show that the data source selection problem for Bayesian learning is NP-hard.  We then show that the data source selection problem can be transformed into an instance of the submodular set covering problem studied in the literature, and provide a standard greedy algorithm to solve the data source selection problem with provable performance guarantees. Next, we propose a fast greedy algorithm that improves the running times of the standard greedy algorithm, while achieving performance guarantees that are comparable to those of the standard greedy algorithm. The fast greedy algorithm can also be applied to solve the general submodular set covering problem with performance guarantees. Finally, we validate the theoretical results using numerical examples, and show that the greedy algorithms work well in practice.

\end{abstract}

\section{Introduction}
The problem of learning the true state of the world based on streams of data has been studied by researchers from different fields. A classical method to tackle this task is Bayesian learning, where we start with a prior belief about the true state of the world and update our belief based on the data streams from the data sources (e.g., \cite{gelman2013bayesian}). In particular, the data streams can come from a variety of sources, including experiment outcomes \cite{chaloner1995bayesian}, medical tests \cite{kononenko1993inductive}, and sensor measurements \cite{krause2008near}, etc. In practice, we need to pay a cost in order to obtain the data streams from the data sources; for example, conducting certain experiments or installing a particular sensor incurs some cost that depends on the nature of the corresponding data source. Thus, a fundamental problem that arises in Bayesian learning is to select a subset of data sources with the smallest total cost, while ensuring a certain level of the learning performance based on the data streams provided by the selected data sources.

In this paper, we focus on a standard Bayesian learning rule that updates the belief on the true state of the world recursively based on the data streams. The learning performance is then characterized by an error given by the difference between the steady-state belief obtained from the learning rule and the true state of the world. Moreover, we consider the scenario where the data sources are selected a priori before running the Bayesian learning rule, and the set of selected data sources is fixed over time. We then formulate and study the Bayesian Learning Data Source Selection (BLDS) problem, where the goal is to minimize the cost spent on the selected data sources while ensuring that the error of the learning process is within a prescribed range. 

\subsection{Related Work}
In \cite{dasgupta2005analysis} and \cite{golovin2010near}, the authors studied the data source selection problem for Bayesian active learning. They considered the scenario  where the data sources are selected in a sequential manner with a single data source  selected at each time step in the learning process. The goal is then to find a policy on sequentially selecting the data sources with minimum cost, while the true state of the world can be identified based on the selected data sources. In contrast, we consider the scenario where a subset of data sources are selected a priori. Moreover, the selected data sources may not necessarily lead to the learning of the true state of the world. Thus, we characterize the performance of the learning process via its steady-state error.

The problem studied in this paper is also related but different from the problem of ensuring sparsity in learning, where the goal is to identify the fewest number of features in order to explain a {\it given} set of data \cite{palmer2004perspectives,krause2010submodular}.

Finally, our problem formulation is also related to the sensor placement problem that has been studied for control systems (e.g., \cite{mo2011sensor} and \cite{ye2018complexity}), signal processing (e.g., \cite{chepuri2014sparsity} and \cite{ye2019sensor}), and machine learning (e.g., \cite{krause2008near}).  In general, the goal of these problems is either to optimize certain (problem-specific) performance metrics of the estimate associated with the measurements of the placed sensors while satisfying the sensor placement budget constraint, or to minimize the cost spent on the placed sensors while ensuring that the estimation performance is within a certain range.

\subsection{Contributions}
First, we formulate the Bayesian Learning Data Source Selection (BLDS) problem, and show that the BLDS problem is NP-hard. Next, we show that the BLDS problem can be transformed into an instance of the submodular set covering problem studied in \cite{wolsey1982analysis}. The BLDS problem can then be solved using a standard greedy algorithm with approximation (i.e., performance) guarantees, where the query complexity of the greedy algorithm is $O(n^2)$, with $n$ to be the number of all candidate data sources. In order to improve the running times of the greedy algorithm, we further propose a fast greedy algorithm with query complexity $O(\frac{n}{\epsilon}\ln\frac{n}{\epsilon})$, where $\epsilon\in(0,1)$. The fast greedy algorithm achieves comparable performance guarantees to those of the standard greedy algorithm, and can also be applied to solve the general submodular set covering problem with performance guarantees. Finally, we provide illustrative examples to interpret the performance bounds obtained for the greedy algorithms applied to the BLDS problem, and give simulation results.

\subsection{Notation and Terminology}
The sets of integers and real numbers are denoted as $\mathbb{Z}$ and $\mathbb{R}$, respectively. For a vector $x\in\mathbb{R}^n$, we denote its transpose as $x^{\prime}$. For $x\in\mathbb{R}$, let $\lceil x\rceil$ be the smallest integer that is greater than or equal to $x$. Given any integer $n\ge1$, we define $[n]\triangleq\{1,\dots,n\}$. The cardinality of a set $\mathcal{A}$ is denoted by $|\mathcal{A}|$. Given two functions $\varphi_1:\mathbb{R}_{\ge0}\to\mathbb{R}$ and $\varphi_2:\mathbb{R}_{\ge0}\to\mathbb{R}$, $\varphi_1(n)$ is $O(\varphi_2(n))$ if there exist positive constants $c$ and $N$ such that $|\varphi_1(n)|\le c|\varphi_2(n)|$ for all $n\ge N$.

\section{The Bayesian Learning Data Source Selection Problem}
\label{sec:problem formulation for Bayesian}
In this section, we formulate the data source selection problem for Bayesian learning that we will study in this paper. Let $\Theta\triangleq\{\theta_1,\theta_2,\dots,\theta_m\}$ be a finite set of possible states of the world, where $m\triangleq|\Theta|$. We consider a set $[n]$ of data sources that can provide data streams of the state of the world. At each discrete time step $k\in\mathbb{Z}_{\ge1}$, the signal (or observation) provided by source $i\in[n]$ is denoted as $\omega_{i,k}\in S_i$, where $S_i$ is the signal space of source $i$. Conditional on the state of the world $\theta\in\Theta$, an observation profile of the $n$ sources at time $k$, denoted as $\omega_k\triangleq(\omega_{1,k},\dots,\omega_{n,k})\in S_1\times\cdots\times S_n$, is generated by the likelihood function $\ell(\cdot|\theta)$. Let $\ell_i(\cdot|\theta)$ denote the $i$-th marginal of $\ell(\cdot|\theta)$, which is the signal structure of data source $i\in[n]$. We make the following assumption on the observation model (e.g., see \cite{jadbabaie2012non,liu2014social,lalitha2014social,nedic2017fast}).
\begin{assumption}
\label{assump:likelihood function}
For each source $i\in[n]$, the signal space $S_i$ is finite, and the likelihood function $\ell_i(\cdot|\theta)$ satisfies $l_i(s_i|\theta)>0$ for all $s_i\in S_i$ and for all $\theta\in\Theta$. Furthermore, for all $\theta\in\Theta$, the observations are independent over time, i.e.,  $\{\omega_{i,1},\omega_{i,2},\dots\}$ is a sequence of independent identically distributed (i.i.d.) random variables. The likelihood function is assumed to satisfy $\ell(\cdot|\theta)=\prod_{i=1}^n\ell_i(\cdot|\theta)$ for all $\theta\in\Theta$, where $\ell_i(\cdot|\theta)$ is the $i$-th marginal of $\ell(\cdot|\theta)$.
\end{assumption}

Consider the scenario where there is a (central) designer who needs to select a subset of data sources in order to learn the true state of the world based on the observations from the selected sources. Specifically, each data source $i\in [n]$ is assumed to have an associated selection cost $h_i\in\mathbb{R}_{>0}$. Considering any $\mathcal{I}\triangleq\{n_1,n_2,\dots,n_{\tau}\}$ with $\tau=|\mathcal{I}|$, we let $h(\mathcal{I})$ denote the sum of the costs of the selected sources in $\mathcal{I}$, i.e., $h(\mathcal{I})\triangleq\sum_{n_i\in\mathcal{I}}h_{n_i}$. Let $\omega_{\mathcal{I},k}\triangleq(\omega_{n_1,k},\dots,\omega_{n_{\tau},k})\in S_{n_1}\times\cdots\times S_{n_{\tau}}$ be the observation profile (conditioned on $\theta\in\Theta$) generated by the likelihood function $\ell_{\mathcal{I}}(\cdot|\theta)$, where $\ell_{\mathcal{I}}(\cdot|\theta)=\prod_{i=1}^{\tau}\ell_{n_i}(\cdot|\theta)$. We assume that the designer knows $\ell_i(\cdot|\theta)$ for all $\theta\in\Theta$ and for all $i\in[n]$, and thus knows $\ell_{\mathcal{I}}(\cdot|\theta)$ for all $\mathcal{I}\subseteq [n]$ and for all $\theta\in\Theta$. After the data sources are selected, the designer updates its belief of the state of the world using the following standard Bayes' rule:
\begin{equation}
\label{eqn:bayesian update}
\mu^{\mathcal{I}}_{k+1}(\theta)=\frac{\mu_0(\theta)\prod_{j=0}^k\ell_{\mathcal{I}}(\omega_{\mathcal{I},j+1}|\theta)}{\sum_{\theta_p\in\Theta}\mu_0(\theta_p)\prod_{j=0}^k\ell_{\mathcal{I}}(\omega_{\mathcal{I},j+1}|\theta_p)}\ \forall\theta\in\Theta,
\end{equation}
where $u^{\mathcal{I}}_{k+1}(\theta)$ is the belief of the designer that $\theta$ is the true state at time step $k+1$, and $\mu_0(\theta)$ is the initial (or prior) belief of the designer that $\theta$ is the true state. We take $\sum_{\theta\in\Theta}\mu_0(\theta)=1$ and $\mu_0(\theta)\in\mathbb{R}_{\ge0}$ for all $\theta\in\Theta$. Note that $\sum_{\theta\in\Theta}\mu^{\mathcal{I}}_k(\theta)=1$ for all $\mathcal{I}\subseteq[n]$ and for all $k\in\mathbb{Z}_{\ge1}$, where $0\le\mu_k^{\mathcal{I}}(\theta)\le1$ for all $\theta\in\Theta$. In other words, $\mu_k^{\mathcal{I}}(\cdot)$ is a probability distribution over $\Theta$ for all $k\in\mathbb{Z}_{\ge1}$ and for all $\mathcal{I}\subseteq[n]$. Rule \eqref{eqn:bayesian update} is also equivalent to the following recursive rule:
\begin{equation}
\label{eqn:bayesian update 2}
\mu^{\mathcal{I}}_{k+1}(\theta)=\frac{\mu^{\mathcal{I}}_k(\theta)\ell_{\mathcal{I}}(\omega_{\mathcal{I},k+1}|\theta)}{\sum_{\theta_p\in\Theta}\mu^{\mathcal{I}}_k(\theta_p)\ell_{\mathcal{I}}(\omega_{\mathcal{I},k+1}|\theta_p)}\ \forall\theta\in\Theta,
\end{equation}
with $\mu_0^{\mathcal{I}}(\theta)\triangleq\mu_0(\theta)$ for all $\mathcal{I}\subseteq[n]$. For a given state $\theta\in\Theta$ and a given $\mathcal{I}\subseteq[n]$, we define the set of {\it observationally equivalent} states to $\theta$ as 
\begin{equation}\label{eqn:def of F_theta 0}
F_{\theta}(\mathcal{I})\triangleq\mathop{\arg\min}_{\theta_p\in\Theta}D_{KL}(\ell_{\mathcal{I}}(\cdot|\theta_p)\|\ell_{\mathcal{I}}(\cdot|\theta)),
\end{equation}
where $D_{KL}(\ell_{\mathcal{I}}(\cdot|\theta_p)\|\ell_{\mathcal{I}}(\cdot|\theta))$ is the Kullback-Leibler (KL) divergence between the likelihood functions $\ell_{\mathcal{I}}(\cdot|\theta_p)$ and $\ell_{\mathcal{I}}(\cdot|\theta)$.  Noting that $D_{KL}(\ell_{\mathcal{I}}(\cdot|\theta)\|\ell_{\mathcal{I}}(\cdot|\theta))=0$ and that the KL divergence is always nonnegative, we have $\theta\in F_{\theta}(\mathcal{I})$ for all $\theta\in\Theta$ and for all $\mathcal{I}\subseteq [n]$. Equivalently, we can write $F_{\theta}(\mathcal{I})$ as
\begin{equation}
\label{eqn:def of F_theta}
F_{\theta}(\mathcal{I})=\{\theta_p\in\Theta:\ \ell_{\mathcal{I}}(s_{\mathcal{I}}|\theta_p)=\ell_{\mathcal{I}}(s_{\mathcal{I}}|\theta),\forall s_{\mathcal{I}}\in S_{\mathcal{I}}\},
\end{equation}
where $S_{\mathcal{I}}\triangleq S_{n_1}\times\cdots\times S_{n_\tau}$. Note that $F_{\theta}(\mathcal{I})$ is the set of states that cannot be distinguished from $\theta$ based on the data streams provided by the data sources indicated by $\mathcal{I}$. Moreover, we define $F_{\theta}(\emptyset)\triangleq\Theta$. Noting that $\ell_{\mathcal{I}}(\cdot|\theta)=\prod_{i=1}^{\tau}\ell_{n_i}(\cdot|\theta)$ under Assumption $\ref{assump:likelihood function}$, we can further obtain from Eqs.~\eqref{eqn:def of F_theta 0}-\eqref{eqn:def of F_theta} the following:
\begin{equation}
\label{eqn:def of F_bar}
F_{\theta}(\mathcal{I})=\bigcap_{n_i\in\mathcal{I}}F_{\theta}(n_i),
\end{equation}
for all $\mathcal{I}\subseteq[n]$ and for all $\theta\in\Theta$. Using similar arguments to those for Lemma $1$ in \cite{mitra2020new}, one can show the following result.
\begin{lemma}
\label{lemma:convergence of bayesian rule}
Suppose the true state of the world is $\theta^*$, and $\mu_0(\theta)>0$ for all $\theta\in\Theta$. For all $\mathcal{I}\subseteq [n]$, the rule given in \eqref{eqn:bayesian update} ensures: (a)  $\mathop{\lim}_{k\to\infty}\mu^{\mathcal{I}}_k(\theta_p)=0$ almost surely (a.s.) for all $\theta_p\notin F_{\theta^*}(\mathcal{I})$; and (b) $\lim_{k\to\infty}\mu^{\mathcal{I}}_k(\theta_q)=\frac{\mu_0(\theta_q)}{\sum_{\theta\in F_{\theta^*}(\mathcal{I})}\mu_0(\theta)}$ a.s. for all $\theta_q\in\ F_{\theta^*}(\mathcal{I})$, where $F_{\theta^*}(\mathcal{I})$ is given by Eq.~\eqref{eqn:def of F_bar}.
\end{lemma}

Consider a true state $\theta^*\in\Theta$ and a set $\mathcal{I}\subseteq [n]$ of selected sources. In order to characterize the (steady-state) learning performance of rule \eqref{eqn:bayesian update}, we will use the following error metric (e.g., \cite{jadbabaie2013information}):
\begin{equation}
e_{\theta^*}(\mathcal{I}) \triangleq \frac{1}{2}\mathop{\lim}_{k\to\infty} \|\mu^{\mathcal{I}}_k - \mathbf{1}_{\theta^*}\|_1,
\end{equation}
where $\mu^{\mathcal{I}}_k\triangleq\begin{bmatrix}\mu^{\mathcal{I}}_k(\theta_1) & \cdots & \mu^{\mathcal{I}}_k(\theta_m)\end{bmatrix}^{\prime}$, and $\mathbf{1}_{\theta^*}\in\mathbb{R}^m$ is a (column) vector where the element that corresponds to $\theta^*$ is $1$ and all the other elements are zero. Note that $\frac{1}{2}\|\mu^{\mathcal{I}}_k - \mathbf{1}_{\theta^*}\|_1$ is also known as the total variation distance between the two distributions $\mu_k^{\mathcal{I}}$ and $\mathbf{1}_{\theta^*}$ (e.g., \cite{bremaud2013markov}). Also note that $e_{\theta^*}(\mathcal{I})$ exists (a.s.) due to Lemma~$\ref{lemma:convergence of bayesian rule}$. We then see from Lemma~$\ref{lemma:convergence of bayesian rule}$ that $e_{\theta^*}(\mathcal{I})=1-\frac{\mu_0(\theta^*)}{\sum_{\theta\in F_{\theta^*}(\mathcal{I})}\mu_0(\theta)}$ holds almost surely. Since the true state is not known a priori to the designer, we further define
\begin{equation}
\label{eqn:error of bayesian case}
e^s_{\theta_p}(\mathcal{I})\triangleq1-\frac{\mu_0(\theta_p)}{\sum_{\theta\in F_{\theta_p}(\mathcal{I})}\mu_0(\theta)}\ \forall\theta_p\in\Theta,
\end{equation}
which represents the (steady-state) total variation distance between the designer's belief $\mu_k^{\mathcal{I}}$ and $\mathbf{1}_{\theta_p}$, when $\theta_p$ is assumed to be the true state of the world. We then define the Bayesian Learning Data Source Selection (BLDS) problem as follows.
\begin{problem}
\label{problem:SSL}
(BLDS) Consider a set $\Theta=\{\theta_1,\dots,\theta_m\}$ of possible states of the world; a set $[n]$ of data sources providing data streams, where the signal space of source $i\in[n]$ is $S_i$ and the observation from source $i\in [n]$ under state $\theta\in\Theta$ is generated by $\ell_i(\cdot|\theta)$; a selection cost $h_i\in\mathbb{R}_{>0}$ of each source $i\in [n]$; an initial belief $\mu_0(\theta)\in\mathbb{R}_{>0}$ for all $\theta\in\Theta$ with $\sum_{\theta\in\Theta}\mu_0(\theta)=1$; and prescribed error bounds $0\le R_{\theta_p}\le 1$ ($R_{\theta_p}\in\mathbb{R}$) for all $\theta_p\in\Theta$. The BLDS problem is to find a set of selected data sources $\mathcal{I}\subseteq[n]$ that solves
\begin{equation}
\label{eqn:problem 1 obj}
\begin{split}
&\mathop{\min}_{\mathcal{I}\subseteq [n]} h(\mathcal{I})\\
& s.t. \ e^s_{\theta_p}(\mathcal{I})\le R_{\theta_p}\ \forall\theta_p\in\Theta,
\end{split}
\end{equation}
where $e^s_{\theta_p}(\mathcal{I})$ is defined in \eqref{eqn:error of bayesian case}.
\end{problem}

Note that the constraints in \eqref{eqn:problem 1 obj} also capture the fact that the true state of the world is unknown to the designer a priori. In other words, for any set $\mathcal{I}\subseteq[n]$ and for any $\theta_p\in\Theta$, the constraint $e^s_{\theta_p}(\mathcal{I})\le R_{\theta_p}$ requires the (steady-state) learning error $e_{\theta_p}^s(\mathcal{I})$ to be upper bounded by $R_{\theta_p}$ when the true state of the world is assumed to be $\theta_p$. Moreover, the interpretation of $R_{\theta_p}$ for $\theta_p\in\Theta$ is as follows. When $R_{\theta_p}=0$, we see from \eqref{eqn:error of bayesian case} and the constraint $e_{\theta_p}^s(\mathcal{I})\le R_{\theta_p}$ that $F_{\theta_p}(\mathcal{I})=\{\theta_p\}$. In other words, the constraint $e_{\theta_p}^s(\mathcal{I})\le 0$ requires that any $\theta_q\in\Theta\setminus\{\theta_p\}$ is {\it not} observationally equivalent to $\theta_p$, based on the observations from the data sources indicated by $\mathcal{I}\subseteq[n]$. Next, when $R_{\theta_p}=1$, we know from \eqref{eqn:error of bayesian case} that the constraint $e^s_{\theta_p}(\mathcal{I})\le1$ is satisfied for all $\mathcal{I}\subseteq[n]$. Finally, when $0<R_{\theta_p}<1$ and $\mu_0(\theta)=\frac{1}{m}$ for all $\theta\in\Theta$, where $m=|\Theta|$, we see from \eqref{eqn:error of bayesian case} that the constraint $e^s_{\theta_p}(\mathcal{I})\le R_{\theta_p}$ is equivalent to $|F_{\theta_p}(\mathcal{I})|\le \frac{1}{1-R_{\theta_p}}$, i.e., the number of states that are observationally equivalent to $\theta_p$ should be less than or equal to $\frac{1}{1-R_{\theta_p}}$, based on the observations from the data source indicated by $\mathcal{I}\subseteq[n]$. In summary, the value of $R_{\theta_p}$ in the constraints represents the requirements of the designer on distinguishing state $\theta_p$ from other states in $\Theta$, where a smaller value of $R_{\theta_p}$ would imply that the designer wants to distinguish $\theta_p$ from more states in $\Theta$ and vice versa. Supposing $R_{\theta_p}=R$ for all $\theta_p\in\Theta$, where $0\le R\le 1$ and $R\in\mathbb{R}$, we see that the constraints in  \eqref{eqn:problem 1 obj} can be equivalently written as $\mathop{\max}_{\theta_p\in\Theta}e_{\theta_p}^s(\mathcal{I})\le R$. 

\begin{remark}
The problem formulation that we described above can be extended to the scenario where the data sources are distributed among a set of agents, and the agents {\it collaboratively} learn the true state of the world using their own observations and communications with other agents. This scenario is known as distributed non-Bayesian learning (e.g., \cite{nedic2017fast}). The goal of the (central) designer is then to select a subset of all the agents whose data sources will be used to collect observations such that the learning error of all the agents is within a prescribed range. More details about this extension can be found in the Appendix.
\end{remark}

Next, we show that the BLDS problem is NP-hard via a reduction from the set cover problem defined in Problem~$\ref{problem:set cover}$, which is known to be NP-hard (e.g., \cite{garey1979computers}, \cite{feige1998threshold}).
\begin{problem}
\label{problem:set cover}
(Set Cover) Consider a set $U=\{u_1,\dots,u_d\}$ and a collection of subsets of $U$, denoted as $\mathcal{C}=\{C_1,\dots,C_k\}$. The set cover problem is to select as few as possible subsets from $\mathcal{C}$ such that every element in $U$ is contained in at least one of the selected subsets.
\end{problem}

\begin{theorem}
\label{thm:SSL is NP-hard}
The BLDS problem is NP-hard even when all the data sources have the same cost, i.e., $h_i=1$ for all $i\in[n]$.
\end{theorem}

\begin{proof}
We give a polynomial-time reduction from the set cover problem to the BLDS problem. Consider an arbitrary instance of the set cover problem as described in Problem $\ref{problem:set cover}$, with the set $U=\{u_1,\dots,u_d\}$ and the collection $\mathcal{C}=\{C_1,\dots,C_k\}$, where $C_i$'s are subsets of $U$. Denote $C_i=\{u_{i_1},\dots,u_{i_{\beta_i}}\}$ for all $i\in [k]$, where $\beta_i=|C_i|$. We then construct an instance of the BLDS problem as follows. The set of possible states of the world is set to be $\Theta=\{\theta_1,\dots,\theta_{d+1}\}$. The number of data sources is set as $n=k$, where the signal space of source $i$ is set to be $S_i=\{0,1\}$ for all $i\in[k]$. For any source $i\in[k]$, the likelihood function $\ell_i(\cdot|\theta)$ corresponding to source $i\in [k]$ is set to satisfy that $\ell_i(0|\theta_1)=\ell_i(1|\theta_1)=\frac{1}{2}$, $\ell_i(0|\theta_{q+1})=\ell_i(1|\theta_{q+1})=\frac{1}{2}$ for all $u_q\in U\setminus C_i$, and $\ell_i(0|\theta_{i_j+1})=\frac{1}{3}$ and $\ell_i(1|\theta_{i_j+1})=\frac{2}{3}$ for all $u_{i_j}\in C_i$. The selection cost is set as $h_i=1$ for all $i\in[k]$. The initial belief is set to be $\mu_0(\theta_p)=\frac{1}{d+1}$ for all $p\in [d+1]$. The prescribed error bounds are set as $R_{\theta_1}=0$ and $R_{\theta_p}=1$ for all $p\in\{2,\dots,d+1\}$. Note that the set of selected sources is denoted as $\mathcal{I}=\{n_1,\dots,n_{\tau}\}\subseteq [k]$.

Since $R_{\theta_p}=1$ for all $p\in\{2,\dots,d+1\}$, the constraint $e^s_{\theta_p}(\mathcal{I})\le R_{\theta_p}$ is satisfied for all $\mathcal{I}\subseteq[n]$ and for all $p\in\{2,\dots,d+1\}$. We then focus on the constraint corresponding to $\theta_1$. Letting $R_{\theta_1}=0$ and $\mu_0(\theta_p)=\frac{1}{d+1}$ for all $p\in[d+1]$, the constraint $e^s_{\theta_1}(\mathcal{I})\le R_{\theta_1}$ is equivalent to $|F_{\theta_1}(\mathcal{I})|\le 1$, where $F_{\theta_1}(\mathcal{I})=\bigcap_{n_i\in\mathcal{I}}F_{\theta_1}(n_i)$ with $F_{\theta_1}(n_i)$ given by Eq.~\eqref{eqn:def of F_theta}. Denote $F_{\theta_1}^c(i)\triangleq\Theta\setminus F_{\theta_1}(i)$ for all $i\in[k]$. From the way we set the likelihood function $\ell_i(\cdot|\theta)$ for source $i\in[k]$ in the constructed instance of the BLDS problem, we see that $F_{\theta_1}^c(i)=\{\theta_{i_1+1},\dots,\theta_{i_{\beta_i}+1}\}$ for all $i\in[k]$, i.e., $C_i\in\mathcal{C}$ corresponds to $F_{\theta_1}^c(i)$ for all $i\in[k]$.  Moreover, using De Morgan's laws, we have
\begin{equation}
\label{eqn:de morgan law}
F_{\theta_1}(\mathcal{I})=\bigcap_{n_i\in\mathcal{I}
}F_{\theta_1}(n_i)=\Theta\setminus\big(\bigcup_{n_i\in\mathcal{I}}F_{\theta_1}^c(n_i)\big).
\end{equation}

Considering any $\mathcal{I}=\{n_1,\dots,n_{\tau}\}\subseteq[k]$ where $\tau=|\mathcal{I}|$, we denote $\mathcal{C}_{\mathcal{I}}\triangleq\{C_{n_1},\dots,C_{n_{\tau}}\}$. We will show that $\mathcal{I}$ is a feasible solution to the given set cover instance (i.e., for any $u_j\in U$, there exists $C_i\in\mathcal{C}_{\mathcal{I}}$ such that $u_j\in C_i$) if and only if $\mathcal{I}$ is a feasible solution to the constructed BLDS instance (i.e., the constraint $e^s_{\theta_1}(\mathcal{I})\le R_{\theta_1}$ is satisfied).

Suppose $\mathcal{I}$ is a feasible solution to the given set cover instance. Since $C_i\in\mathcal{C}$ corresponds to $F_{\theta_1}^c(i)$ for all $i\in[k]$, we see that for any $\theta_p\in\{\theta_2,\dots,\theta_{d+1}\}$, there exists $n_i\in\mathcal{I}$ such that $\theta_p\in F_{\theta_1}^c(n_i)$ in the constructed BLDS instance, which implies $\bigcup_{n_i\in\mathcal{I}}F_{\theta_1}^c(n_i)=\{\theta_2,\dots,\theta_{d+1}\}$. It follows from \eqref{eqn:de morgan law} that $F_{\theta_1}(\mathcal{I})=\Theta\setminus\{\theta_2,\dots,\theta_{d+1}\}=\{\theta_1\}$, which implies that the constraint $|F_{\theta_1}(\mathcal{I})|\le 1$ is satisfied, i.e., the constraint $e^s_{\theta_1}(\mathcal{I})\le R_{\theta_1}$ is satisfied. Conversely, suppose $\mathcal{I}$ is a feasible solution to the constructed BLDS instance, i.e., the constraint $e^s_{\theta_1}(\mathcal{I})\le R_{\theta_1}$ is satisfied, which implies $|F_{\theta_1}(\mathcal{I})|\le 1$. Noting that $\theta_1\in F_{\theta_1}(\mathcal{I})$ for all $\mathcal{I}\subseteq [k]$, we have $F_{\theta_1}(\mathcal{I})=\{\theta_1\}$. We then see from \eqref{eqn:de morgan law} that $\bigcup_{n_i\in\mathcal{I}}F_{\theta_1}^c(n_i)=\{\theta_2,\dots,\theta_{d+1}\}$, i.e., for all $\theta_p\in\{\theta_2,\dots,\theta_{d+1}\}$, there exists $n_i\in\mathcal{I}$ such that $\theta_p\in F_{\theta_1}^c(n_i)$. It then follows from the one-to-one correspondence between $C_i$ and $F_{\theta_1}^c(i)$ that  for any $u_j\in U$, there exists $C_{n_i}\in\mathcal{C}_{\mathcal{I}}$ such that $u_j\in C_{n_i}$ in the set cover instance. 

Since the selection cost is set as $h_i=1$ for all $i\in[k]$, we see from the above arguments that $\mathcal{I}^*$ is an optimal solution to the set cover instance if and only if it is an optimal solution to the BLDS instance. Since the set cover problem is NP-hard, we conclude that the BLDS problem is NP-hard.
\end{proof}

\section{Submodularity and Greedy Algorithms for the BLDS Problem}
\label{sec:greedy for SPL}
In this section, we first show that the BLDS problem can be transformed into an instance of the submodular set covering problem studied in \cite{wolsey1982analysis}. We then consider two greedy algorithms for the BLDS problem and study their performance guarantees when applied to the problem. We start with the following definition.

\begin{definition}
\label{def:submodular}
(\cite{nemhauser1978analysis}) A set function $f:2^{[n]}\to\mathbb{R}$ is submodular if for all $X\subseteq Y\subseteq[n]$ and for all $j\in [n]\setminus Y$,
\begin{equation}
\label{eqn:submodular def 1}
f(X\cup\{j\})-f(X)\ge f(Y\cup\{j\})-f(Y).
\end{equation}
Equivalently, $f:2^{[n]}\to\mathbb{R}$ is submodular if for all $X,Y\subseteq[n]$,
\begin{equation}
\label{eqn:submodular def 2}
\sum_{j\in Y\setminus X}(f(X\cup\{j\})-f(X))\ge f(Y\cup X)-f(X).
\end{equation}
\end{definition}

To proceed, note that the constraint corresponding to $\theta_p$ in Problem $\ref{problem:SSL}$ (i.e., \eqref{eqn:problem 1 obj}) is satisfied for all $\mathcal{I}\subseteq[n]$ if $R_{\theta_p}=1$. Since $\mu_0(\theta)>0$ for all $\theta\in\Theta$, we can then equivalently write the constraints as
\begin{equation}
\label{eqn:rewritten constraint for SSL 1}
\sum_{\theta\in F_{\theta_p}(\mathcal{I})}\mu_0(\theta)\le\frac{\mu_0(\theta_p)}{1-R_{\theta_p}},\ \forall\theta_p\in\Theta\ \text{with}\ R_{\theta_p}<1.
\end{equation}
Define $F_{\theta}^c(\mathcal{I})\triangleq\Theta\setminus F_{\theta}(\mathcal{I})$ for all $\theta\in\Theta$ and for all $\mathcal{I}\subseteq [n]$, where $F_{\theta}(\mathcal{I})$ is given by Eq.~\eqref{eqn:def of F_bar}. Note that $F_{\theta}^c(\mathcal{I})$ is the set of states that can be distinguished from $\theta$, given the data sources indicated by $\mathcal{I}$. Using the fact $\sum_{\theta\in\Theta}\mu_0(\theta)=1$, \eqref{eqn:rewritten constraint for SSL 1} can be equivalently written as
\begin{equation}
\label{eqn:rewritten constraint for SSL 2}
\sum_{\theta\in F_{\theta_p}^c(\mathcal{I})}\mu_0(\theta)\ge 1-\frac{\mu_0(\theta_p)}{1-R_{\theta_p}},\ \forall\theta_p\in\Theta\ \text{with}\ R_{\theta_p}<1.
\end{equation}
Moreover, we note that the constraint corresponding to $\theta_p$ in \eqref{eqn:rewritten constraint for SSL 2} is satisfied for all $\mathcal{I}\subseteq[n]$ if $1-\frac{\mu_0(\theta_p)}{1-R_{\theta_p}}\le 0$, i.e., $R_{\theta_p}\ge1-\mu_0(\theta_p)$.  Hence, we can equivalently write \eqref{eqn:rewritten constraint for SSL 2} as 
\begin{equation*}
\label{eqn:rewritten constraint for SSL 3}
\sum_{\theta\in F_{\theta_p}^c(\mathcal{I})}\mu_0(\theta)\ge 1-\frac{\mu_0(\theta_p)}{1-R_{\theta_p}},\ \forall\theta_p\in\bar{\Theta},
\end{equation*}
where $\bar{\Theta}\triangleq\{\theta_p\in\Theta: 0\le R_{\theta_p}<1-\mu_0(\theta_p)\}$. 
For all $\mathcal{I}\subseteq [n]$, let us define
\begin{equation}
\label{eqn:def of f}
 f_{\theta_p}(\mathcal{I})\triangleq\sum_{\theta\in F_{\theta_p}^c(\mathcal{I})}\mu_0(\theta),\ \forall\theta_p\in\bar{\Theta}.
 \end{equation}
Noting that $F_{\theta_p}(\emptyset)=\Theta$, i.e., $F_{\theta_p}^c(\emptyset)=\emptyset$, we  let $f_{\theta_p}(\emptyset)=0$. It then follows directly from \eqref{eqn:def of f} that $f_{\theta_p}:2^{[n]}\to\mathbb{R}_{\ge0}$ is a monotone nondecreasing set function.\footnote{A set function $f:2^{[n]}\to\mathbb{R}$ is monotone nondecreasing if $f(X)\le f(Y)$ for all $X\subseteq Y\subseteq [n]$.}

\begin{remark}
\label{remark:n sensors constraints}
Note that in order to ensure that there exists $\mathcal{I}\subseteq [n]$ that satisfies the constraints in \eqref{eqn:rewritten constraint for SSL 2}, we assume that $f_{\theta_p}([n])\ge 1-\frac{\mu_0(\theta_p)}{1-R_{\theta_p}}$ for all $\theta_p\in\bar{\Theta}$, since $f_{\theta_p}(\cdot)$ is nondecreasing for all $\theta_p\in\bar{\Theta}$.
\end{remark}

\begin{lemma}
\label{lemma:submodular of f prime}
The set function $f_{\theta_p}:2^{[n]}\to\mathbb{R}_{\ge0}$ defined in \eqref{eqn:def of f} is submodular for all $\theta_p\in\bar{\Theta}$.
\end{lemma}
\begin{proof}
Consider any $\mathcal{I}_1\subseteq\mathcal{I}_2\subseteq[n]$ and any $j\in[n]\setminus\mathcal{I}_2$. For all $\mathcal{I}\subseteq[n]$, we will drop the dependency of $F_{\theta_p}(\mathcal{I})$ (resp., $F_{\theta_p}^c(\mathcal{I})$) on $\theta_p$, and write $F(\mathcal{I})$ (resp., $F^c(\mathcal{I})$) for notational simplicity in this proof. We then have the following:
\begin{align}\nonumber
&f_{\theta_p}(\mathcal{I}_1\cup\{j\})-f_{\theta_p}(\mathcal{I}_1)\\\nonumber
=&\sum_{\theta\in F^c(\mathcal{I}_1\cup\{j\})}\mu_0(\theta)-\sum_{\theta\in F^c(\mathcal{I}_1)}\mu_0(\theta)\label{eqn:derive submodular 1}\\
=&\sum_{\theta\in F^c(\mathcal{I}_1)\cup F^c(j)}\mu_0(\theta)-\sum_{\theta\in F^c(\mathcal{I}_1)}\mu_0(\theta)\\
=&\sum_{\theta\in(F^c(\mathcal{I}_1)\cup F^c(j))\setminus F^c(\mathcal{I}_1)}\mu_0(\theta)=\sum_{\theta\in F^c(j)\setminus F^c(\mathcal{I}_1)}\mu_0(\theta)\label{eqn:derive submodular 2}.
\end{align}
To obtain \eqref{eqn:derive submodular 1}, we note $F^c(\mathcal{I}_1\cup\{j\})=\Theta\setminus F(\mathcal{I}_1\cup\{j\})=\Theta\setminus(F(\mathcal{I}_1)\cap F(j))$, which implies (via De Morgan's laws) $F^c(\mathcal{I}_1\cup\{j\})=F^c(\mathcal{I}_1)\cup F^c(j)$. Similarly, we also have
\begin{equation}
f_{\theta_p}(\mathcal{I}_2\cup\{j\})-f_{\theta_p}(\mathcal{I}_2)=\sum_{\theta\in F^c(j)\setminus F^c(\mathcal{I}_2)}\mu_0(\theta)\label{eqn:derive submodular 3}.
\end{equation}
Since $\mathcal{I}_1\subseteq\mathcal{I}_2$, we have $F^c(j)\setminus F^c(\mathcal{I}_2)\subseteq F^c(j)\setminus F^c(\mathcal{I}_1)$, which implies via \eqref{eqn:derive submodular 2}-\eqref{eqn:derive submodular 3}
\begin{equation*}
f_{\theta_p}(\mathcal{I}_1\cup\{j\})-f_{\theta_p}(\mathcal{I}_1)\ge f_{\theta_p}(\mathcal{I}_2\cup\{j\})-f_{\theta_p}(\mathcal{I}_2).
\end{equation*}
Since the above arguments hold for all $\theta_p\in\bar{\Theta}$, we know from \eqref{eqn:submodular def 1} in Definition $\ref{def:submodular}$ that $f_{\theta_p}(\cdot)$ is submodular for all $\theta_p\in\bar{\Theta}$.
\end{proof}

Moreover, considering any $\mathcal{I}\subseteq[n]$, we define 
\begin{equation}
\label{eqn:def of f prime}
f'_{\theta_p}(\mathcal{I})\triangleq\mathop{\min}\{f_{\theta_p}(\mathcal{I}),1-\frac{\mu_0(\theta_p)}{1-R_{\theta_p}}\}\ \forall\theta_p\in\bar{\Theta},
\end{equation}
where $f_{\theta_p}(\mathcal{I})$ is defined in \eqref{eqn:def of f}. Since $f_{\theta_p}(\cdot)$ is submodular and nondecreasing with $f_{\theta_p}(\emptyset)=0$ and $f_{\theta_p}([n])\ge1-\frac{\mu_0(\theta_p)}{1-R_{\theta_p}}$, one can show that $f'_{\theta_p}(\cdot)$ is also submodular and nondecreasing with $f'_{\theta_p}(\emptyset)=0$ and $f'_{\theta_p}([n])=1-\frac{\mu_0(\theta_p)}{1-R_{\theta_p}}$. Noting that the sum of submodular functions remains submodular, we see that $\sum_{\theta_p\in\bar{\Theta}}f'_{\theta_p}(\cdot)$ is submodular and nondecreasing. We also have the following result.
\begin{lemma}
\label{lemma:equivalent constraints}
Consider any $\mathcal{I}\subseteq [n]$. The constraint $\sum_{\theta\in F_{\theta_p}^c(\mathcal{I})}\mu_0(\theta)\ge 1-\frac{\mu_0(\theta_p)}{1-R_{\theta_p}}$ holds for all $\theta_p\in\bar{\Theta}$ if and only if $\sum_{\theta_p\in\bar{\Theta}}f'_{\theta_p}(\mathcal{I})=\sum_{\theta_p\in\bar{\Theta}}f'_{\theta_p}([n])$, where $f'_{\theta_p}(\cdot)$ is defined in \eqref{eqn:def of f prime}.
\end{lemma}
\begin{proof}
Suppose the constraints $\sum_{\theta\in F_{\theta_p}^c(\mathcal{I})}\mu_0(\theta)\ge 1-\frac{\mu_0(\theta_p)}{1-R_{\theta_p}}$ hold for all $\theta_p\in\bar{\Theta}$. It follows from  \eqref{eqn:def of f prime}  that $f'_{\theta_p}(\mathcal{I})=1-\frac{\mu_0(\theta_p)}{1-R_{\theta_p}}$ for all $\theta_p\in\bar{\Theta}$. Noting that $f_{\theta_p}([n])\ge1-\frac{\mu_0(\theta_p)}{1-R_{\theta_p}}$ as argued in Remark~$\ref{remark:n sensors constraints}$, we have $f'_{\theta_p}([n])=1-\frac{\mu_0(\theta_p)}{1-R_{\theta_p}}$ for all $\theta_p\in\bar{\Theta}$, which implies $\sum_{\theta_p\in\bar{\Theta}}f'_{\theta_p}(\mathcal{I})=\sum_{\theta_p\in\bar{\Theta}}f'_{\theta_p}([n])$. Conversely, suppose $\sum_{\theta_p\in\bar{\Theta}}f'_{\theta_p}(\mathcal{I})=\sum_{\theta_p\in\bar{\Theta}}f'_{\theta_p}([n])$, i.e., $\sum_{\theta_p\in\bar{\Theta}}\Big(f'_{\theta_p}(\mathcal{I})-\big(1-\frac{\mu_0(\theta_p)}{1-R_{\theta_p}}\big)\Big)=0$. Noting from \eqref{eqn:def of f prime} that $f'_{\theta_p}(\mathcal{I})\le1-\frac{\mu_0(\theta_p)}{1-R_{\theta_p}}$ for all $\mathcal{I}\subseteq[n]$, we have $f'_{\theta_p}(\mathcal{I})=1-\frac{\mu_0(\theta_p)}{1-R_{\theta_p}}$ for all $\theta_p\in\bar{\Theta}$, i.e., $f_{\theta_p}(\mathcal{I})\ge1-\frac{\mu_0(\theta_p)}{1-R_{\theta_p}}$ for all $\theta_p\in\bar{\Theta}$. This completes the proof of the lemma.
\end{proof}

Based on the above arguments, for all $\mathcal{I}\subseteq[n]$, we further define
\begin{equation}
\label{eqn:def of z}
z(\mathcal{I})\triangleq\sum_{\theta_p\in\bar{\Theta}}f'_{\theta_p}(\mathcal{I})=\sum_{\theta_p\in\bar{\Theta}}\mathop{\min}\{f_{\theta_p}(\mathcal{I}),1-\frac{\mu_0(\theta_p)}{1-R_{\theta_p}}\},
\end{equation}
where $f_{\theta_p}(\mathcal{I})$ is defined in \eqref{eqn:def of f}. We then see from Lemma~$\ref{lemma:equivalent constraints}$ that \eqref{eqn:problem 1 obj} in Problem~$\ref{problem:SSL}$ can be equivalently written as 
\begin{equation}
\label{eqn:equivalent form of spl 1}
\begin{split}
&\mathop{\min}_{\mathcal{I}\subseteq[n]} h(\mathcal{I})\\
&s.t.\ z(\mathcal{I})=z([n]),
\end{split}
\end{equation}
where one can show that $z(\cdot)$ defined in Eq.~\eqref{eqn:def of z} is a nondecreasing and submodular set function with $z(\emptyset)=0$. Now, considering an instance of the BLDS problem, for any $\mathcal{I}\subseteq[n]$ and for any $\theta\in\Theta$, one can obtain $F_{\theta}(\mathcal{I})$ (and  $F_{\theta}^c(\mathcal{I})$) in $O(S|\mathcal{I}||\Theta|)$ time, where $S\triangleq\mathop{\max}_{n_i\in\mathcal{I}}|S_i|$ with $S_i$ to be the signal space of source $n_i\in\mathcal{I}$. Therefore, we see from \eqref{eqn:def of f} and \eqref{eqn:def of z} that for any $\mathcal{I}\subseteq[n]$, one can compute the value of $z(\mathcal{I})$ in $O(Sn|\Theta|^2)$ time. 

\subsection{Standard Greedy Algorithm}
Problem~\eqref{eqn:equivalent form of spl 1} can now be viewed as the submodular set covering problem studied in \cite{wolsey1982analysis}, where the submodular set covering problem is solved using a standard greedy algorithm with performance guarantees. Specifically, we consider the greedy algorithm defined in Algorithm~$\ref{algorithm:greedy}$ for the BLDS problem. The algorithm maintains a sequence of sets $\mathcal{I}_g^0,\mathcal{I}_g^1,\dots,\mathcal{I}_g^{T}$ containing the selected elements from $[n]$, where $T\in\mathbb{Z}_{\ge1}$. Note that Algorithm~$\ref{algorithm:greedy}$ requires $O(n^2)$ evaluations of function $z(\cdot)$, where $z(\mathcal{I})$ can be computed in $O(Sn|\Theta|^2)$ time for any $\mathcal{I}\subseteq[n]$ as argued above. In other words, the query complexity of Algorithm~$\ref{algorithm:greedy}$ is $O(n^2)$. We then have the following result from the arguments above (i.e., Lemmas~$\ref{lemma:submodular of f prime}$-$\ref{lemma:equivalent constraints}$) and Theorem~$1$ in \cite{wolsey1982analysis}, which characterizes the performance guarantees for the greedy algorithm (Algorithm $\ref{algorithm:greedy}$) when applied to the BLDS problem.

\begin{algorithm}
\textbf{Input:} $[n]$, $z:2^{[n]}\to\mathbb{R}_{\ge0}$, $h_i$ $\forall i\in[n]$\\
\textbf{Output:} $\mathcal{I}_g$
\caption{Greedy Algorithm for BLDS}\label{algorithm:greedy}
\begin{algorithmic}[1]
\State $t\gets 0$, $\mathcal{I}_g^0\gets\emptyset$
\While{$z(\mathcal{I}_g^t)< z([n])$}
    \State $j_t\in\mathop{\arg\max}_{i\in[n]\setminus\mathcal{I}_g^t}\frac{z(\mathcal{I}_g^t\cup\{i\})-z(\mathcal{I}_g^t)}{h_i}$
    \State $\mathcal{I}_g^{t+1}\gets \mathcal{I}_g^t\cup\{j_t\}$, $t\gets t+1$
\EndWhile
\State $T\gets t$, $\mathcal{I}_g\gets\mathcal{I}_g^T$
\State \textbf{return} $\mathcal{I}_g$
\end{algorithmic}
\end{algorithm}

\begin{theorem}
\label{thm:guarantee for greedy SPL}
Let $\mathcal{I}^*$ be an optimal solution to the BLDS problem. Algorithm $\ref{algorithm:greedy}$ returns a solution $\mathcal{I}_g$ to the BLDS problem (i.e., \eqref{eqn:equivalent form of spl 1}) that satisfies the following, where $\mathcal{I}_g^1,\dots,\mathcal{I}_g^{T-1}$ are specified in Algorithm~$\ref{algorithm:greedy}$.

\noindent(a) $h(\mathcal{I}_g)\le \bigg(1+\ln \displaystyle\mathop{\max}_{i\in[n],\zeta\in[T-1]}\Big\{\frac{z(i)-z(\emptyset)}{z(\mathcal{I}_g^{\zeta}\cup\{i\})-z(\mathcal{I}_g^{\zeta})}: z(\mathcal{I}_g^{\zeta}\cup\{i\})-z(\mathcal{I}_g^{\zeta})>0\Big\}\bigg)h(\mathcal{I}^*)$,

\noindent(b) $h(\mathcal{I}_g)\le\Big(1+\ln \frac{h_{j_T}(z(j_1)-z(\emptyset))}{h_{j_1}(z(\mathcal{I}_g^{T-1}\cup\{j_T\})-z(\mathcal{I}_g^{T-1}))}\Big)h(\mathcal{I}^*)$,

\noindent(c) $h(\mathcal{I}_g)\le \Big(1+\ln \frac{z([n])-z(\emptyset)}{z([n])-z(\mathcal{I}_g^{T-1})}\Big)h(\mathcal{I}^*)$,

\noindent(d) if $z(\mathcal{I})\in\mathbb{Z}_{\ge0}$ for all $\mathcal{I}\subseteq[n]$, $h(\mathcal{I}_g)\le\big(\sum_{i=i}^M\frac{1}{i}\big)h(\mathcal{I}^*)\le(1+\ln M)h(\mathcal{I}^*)$, where $M\triangleq\mathop{\max}_{j\in[n]}z(j)$.
\end{theorem}

Note that the bounds in Theorem~$\ref{thm:guarantee for greedy SPL}$(a)-(c) depend on $\mathcal{I}_g^t$ from the greedy algorithm. We can compute the bounds in Theorem~$\ref{thm:guarantee for greedy SPL}$(a)-(c) in parallel with the greedy algorithm, in order to provide a performance guarantee on the output of the algorithm. The bound in Theorem~$\ref{thm:guarantee for greedy SPL}$(d) does not depend on $\mathcal{I}_g^t$, and can be computed using $O(n)$ evaluations of function $z(\cdot)$.

\subsection{Fast greedy algorithm}\label{sec:fast greedy}
We now give an algorithm (Algorithm~$\ref{algorithm:fast greedy}$) for BLDS that achieves $O(\frac{n}{\epsilon}\ln\frac{n}{\epsilon})$ query complexity for any $\epsilon\in(0,1)$, which is significantly smaller than $O(n^2)$ as $n$ scales large. In line $3$ of Algorithm~$\ref{algorithm:fast greedy}$, $h_{\max}\triangleq\mathop{\max}_{j\in[n]}h_j$ and $h_{\min}\triangleq\mathop{\min}_{j\in[n]}h_j$. While achieving faster running times, we will show that the solution returned by Algorithm~$\ref{algorithm:fast greedy}$ has slightly worse performance bounds compared to those of Algorithm~$\ref{algorithm:greedy}$ provided in Theorem~$\ref{thm:guarantee for greedy SPL}$, and potentially slightly violates the constraint of the BLDS problem given in \eqref{eqn:equivalent form of spl 1}. Specifically, a larger value of $\epsilon$ in Algorithm~$\ref{algorithm:fast greedy}$ leads to faster running times of Algorithm~$\ref{algorithm:fast greedy}$, but yields worse performance guarantees. Moreover, note that Algorithm~$\ref{algorithm:greedy}$ adds a single element to $\mathcal{I}_g$ in each iteration of the while loop in lines $2$-$4$. In contrast, Algorithm~$\ref{algorithm:fast greedy}$ considers multiple candidate elements in each iteration of the for loop in lines $3$-$9$, and adds elements that satisfy the threshold condition given in line $5$, which leads to faster running times. Formally, we have the following result. 

\begin{algorithm}
\textbf{Input:} $[n]$, $z:2^{[n]}\to\mathbb{R}_{\ge0}$, $h_i$ $\forall i\in[n]$, $\epsilon\in(0,1)$\\
\textbf{Output:} $\mathcal{I}_f$
\caption{Fast Greedy Algorithm for BLDS}\label{algorithm:fast greedy}
\begin{algorithmic}[1]
\State $t\gets0$, $\mathcal{I}_f^0\gets\emptyset$
\State $d\gets\mathop{\max}_{i\in[n]}\frac{z(i)-z(\emptyset)}{h_i}$
\For{$(\tau=d$; $\tau\ge\frac{\epsilon h_{\mathop{\min}}}{nh_{\mathop{\max}}}d$; $\tau\gets\tau(1-\epsilon))$}
    \For{$j\in[n]$}
        \If{$\frac{z(\mathcal{I}_f^t\cup\{j\})-z(\mathcal{I}_f^t)}{h_j}\ge\tau$}
            \State $\mathcal{I}_f^{t+1}\gets\mathcal{I}_f^t\cup\{j\}$, $t\gets t+1$
        \EndIf
        \If{$z(\mathcal{I}_f^t)=z([n])$}
            \State $T\gets t$, $\mathcal{I}_f\gets\mathcal{I}_f^T$
            \State \textbf{return}\ $\mathcal{I}_f$
        \EndIf
    \EndFor
\EndFor
\State $T\gets t$, $\mathcal{I}_f\gets\mathcal{I}_f^T$
\State\textbf{return} $\mathcal{I}_f$
\end{algorithmic}
\end{algorithm}

\begin{theorem}
\label{thm:fast greedy}
Suppose $\frac{h_{\mathop{\max}}}{h_{\mathop{\min}}}\le n^H$ holds in the BLDS instances, where $h_{\max}=\mathop{\max}_{j\in[n]}h_j$, $h_{\min}=\mathop{\min}_{j\in[n]}h_j$, and $H\in\mathbb{R}_{\ge1}$ is a fixed constant. Let $\mathcal{I}^*$ be an optimal solution to the BLDS problem. For any $\epsilon\in(0,1)$, Algorithm~$\ref{algorithm:fast greedy}$ returns a solution $\mathcal{I}_f$ to the BLDS problem (i.e., \eqref{eqn:equivalent form of spl 1}) in query complexity $O(\frac{n}{\epsilon}\ln\frac{n}{\epsilon})$ that satisfies $z(\mathcal{I}_f)\ge (1-\epsilon)z([n])$, and has the following performance bounds, where $\mathcal{I}_f^{T-1}$ is given in Algorithm~$\ref{algorithm:fast greedy}$.

\noindent(a) $h(\mathcal{I}_f)\le\frac{1}{1-\epsilon}\Big(1+\ln\frac{z([n])}{z([n])-z(\mathcal{I}_f^{T-1})}\Big)h(\mathcal{I}^*)$,

\noindent(b) if $z(\mathcal{I})\in\mathbb{Z}_{\ge0}$ for all $\mathcal{I}\subseteq[n]$, $h(\mathcal{I}_f)\le\frac{1}{1-\epsilon}\big(1+\ln z([n])\big)h(\mathcal{I}^*)$.
\end{theorem}
\begin{proof}
Consider any $\epsilon\in(0,1)$. We first show that the query complexity of Algorithm~$\ref{algorithm:fast greedy}$ is $O(\frac{n}{\epsilon}\ln\frac{n}{\epsilon})$. Note that the for loop in lines $3$-$9$ runs for at most $K_{\mathop{\max}}\triangleq\lceil\frac{1}{-\ln(1-\epsilon)}\cdot(\ln\frac{n}{\epsilon}+\ln\frac{h_{\mathop{\max}}}{h_{\mathop{\min}}})\rceil$ iterations, where each iteration requires $O(n)$ evaluations of $z(\cdot)$. One can also show that $-\ln(1-\epsilon)-\epsilon>0$ for $\epsilon\in(0,1)$, which implies $K_{\mathop{\max}}\le\frac{1}{\epsilon}\cdot(\ln\frac{n}{\epsilon}+H\ln n)+1\le\frac{1}{\epsilon}((H+1)\ln\frac{n}{\epsilon}+1)$, where $H\in\mathbb{R}_{\ge1}$ is a fixed constant. It then follows from the above arguments that the query complexity of Algorithm~$\ref{algorithm:fast greedy}$ is $O(\frac{n}{\epsilon}\ln\frac{n}{\epsilon})$.

Next, we show that $\mathcal{I}_f$ satisfies $z(\mathcal{I}_f)\ge (1-\epsilon)z([n])$. Note that if Algorithm~$\ref{algorithm:fast greedy}$ ends with line $9$, then $z(\mathcal{I}_f)=z([n])$ and thus $z(\mathcal{I}_f)\ge(1-\epsilon)z([n])$ hold. Hence, we assume that Algorithm~$\ref{algorithm:fast greedy}$ ends with $\tau=\frac{\epsilon h_{\mathop{\min}}}{n h_{\mathop{\max}}}d$ in the for loop in lines $3$-$9$. Also note that $z(\emptyset)=0$. Denoting $j^*\in\mathop{\arg}{\max}_{i\in[n]}\frac{z(i)-z(\emptyset)}{h_i}$ and considering any $j\in[n]\setminus\mathcal{I}_f$, we have from the definition of Algorithm~$\ref{algorithm:fast greedy}$ the following:
\begin{align}\nonumber
&\frac{z(\mathcal{I}_f\cup\{j\})-z(\mathcal{I}_f)}{h_j}<\frac{\epsilon h_{\mathop{\min}}z(j^*)}{n h_{\mathop{\max}}h_{j^*}},\\
\implies&z(\mathcal{I}_f\cup\{j\})-z(\mathcal{I}_f)<\frac{\epsilon}{n}z(j^*)\le\frac{\epsilon}{n}z([n]),\label{eqn:single margin}
\end{align}
where we use the facts $h_j\le h_{\mathop{\max}}$ and $h_{j^*}\ge h_{\mathop{\min}}$ to obtain the first inequality in \eqref{eqn:single margin}, and use the fact that $z(\cdot)$ is monotone nondecreasing to obtain the second inequality in \eqref{eqn:single margin}. Since \eqref{eqn:single margin} holds for all  $j\in[n]\setminus\mathcal{I}_f$, it follows that
\begin{equation*}
\sum_{j\in[n]\setminus\mathcal{I}_f}\big(z(\mathcal{I}_f\cup\{j\})-z(\mathcal{I}_f)\big)<\epsilon z([n])
\implies z([n])-z(\mathcal{I}_f)<\epsilon z([n]),
\end{equation*}
where we use the submodularity of $z(\cdot)$ (i.e., \eqref{eqn:submodular def 2} in Definition~$\ref{def:submodular}$).

We now prove part (a). Denote $\mathcal{I}_f^t=\{j_1,\dots,j_t\}\subseteq[n]$ for all $t\in[T]$ with $\mathcal{I}_f^0=\emptyset$ in Algorithm~$\ref{algorithm:fast greedy}$. First, suppose $T\ge2$. Considering any $t\in[T-1]$, we have from line $5$ in Algorithm~$\ref{algorithm:fast greedy}$:
\begin{equation}
\frac{z(\mathcal{I}_f^t\cup\{j_{t+1}\})-z(\mathcal{I}_f^t)}{h_{j_{t+1}}}\ge\tau.\label{eqn:greedy choice 1}
\end{equation}
Moreover, consider any $j\in[n]\setminus\mathcal{I}_f^t$. Since $j$ has not been added to $\mathcal{I}_f^t$ while the current threshold   is $\tau$, one can see that $j$ does not satisfy the threshold condition in line $5$ when the threshold was $\frac{\tau}{1-\epsilon}$, i.e.,
\begin{equation}
\frac{z(\mathcal{I}_f^{t'}\cup\{j\})-z(\mathcal{I}_f^{t'})}{h_j}\le\frac{\tau}{1-\epsilon}\implies\frac{z(\mathcal{I}_f^t\cup\{j\})-z(\mathcal{I}_f^t)}{h_j}\le\frac{\tau}{1-\epsilon}, \label{eqn:greedy choice 2}
\end{equation}
where $t'\in\{0,\dots,T-1\}$ with $t'<t$ is a corresponding time step in Algorithm~$\ref{algorithm:fast greedy}$ when the threshold was $\frac{\tau}{1-\epsilon}$. Note that we obtain the second inequality in \eqref{eqn:greedy choice 2} using again the submodularity of $z([n])$ (i.e., \eqref{eqn:submodular def 1} in Definition~$\ref{def:submodular}$). Combining \eqref{eqn:greedy choice 1} and \eqref{eqn:greedy choice 2}, we have 
\begin{equation}
\frac{z(\mathcal{I}_f^t\cup\{j_{t+1}\})-z(\mathcal{I}_f^t)}{h_{j_{t+1}}}\ge\frac{(1-\epsilon)(z(\mathcal{I}_f^t\cup\{j\})-z(\mathcal{I}_f^t))}{h_j}.\label{eqn:greedy choice 3}
\end{equation} 
Noting that \eqref{eqn:greedy choice 3} holds for all $j\in\mathcal{I}^*\setminus\mathcal{I}_f^t$, one can show that
\begin{equation}
\label{eqn:mean lower bound}
\frac{z(\mathcal{I}_f^t\cup\{j_{t+1}\})-z(\mathcal{I}_f^t)}{h_{j_{t+1}}}\ge\frac{(1-\epsilon)\sum_{j\in\mathcal{I}^*\setminus\mathcal{I}_f^t}(z(\mathcal{I}_f^t\cup\{j\})-z(\mathcal{I}_f^t))}{\sum_{j\in\mathcal{I}^*\setminus\mathcal{I}_f^t}h_j},
\end{equation}
which further implies, via the fact that $z(\cdot)$ is submodular and monotone nondecreasing, the following:
\begin{equation}
\label{eqn:submodular lower bound}
\frac{z(\mathcal{I}_f^t\cup\{j_{t+1}\})-z(\mathcal{I}_f^t)}{h_{j_{t+1}}}\ge\frac{(1-\epsilon)(z(\mathcal{I}^*\cup\mathcal{I}_f^t)-z(\mathcal{I}_f^t))}{h(\mathcal{I}^*\setminus\mathcal{I}_f^t)}\ge\frac{(1-\epsilon)(z(\mathcal{I}^*)-z(\mathcal{I}_f^t))}{h(\mathcal{I}^*)}.
\end{equation}
Rearranging the terms in \eqref{eqn:submodular lower bound}, we have
\begin{align}\nonumber
&z(\mathcal{I}^*)-z(\mathcal{I}_f^t)\le\frac{h(\mathcal{I}^*)}{1-\epsilon}\cdot\frac{z(\mathcal{I}^*)-z(\mathcal{I}_f^t)-(z(\mathcal{I}^*)-z(\mathcal{I}_f^{t+1}))}{h_{j_{t+1}}},\\
\implies&z(\mathcal{I}^*)-z(\mathcal{I}_f^{t+1})\le(1-\frac{(1-\epsilon)h_{j_t+1}}{h(\mathcal{I}^*)})(z(\mathcal{I}^*)-z(\mathcal{I}_f^t)).\label{eqn:recursion}
\end{align}
Moreover, we see from the above arguments that \eqref{eqn:recursion} holds for all $t\in[T-1]$. Now, considering $t=0$ and using similar arguments to those above, we can show that \eqref{eqn:greedy choice 3} and thus \eqref{eqn:recursion} also hold. Therefore, viewing \eqref{eqn:recursion} as a recursion of $z(\mathcal{I}^*)-z(\mathcal{I}_f^t)$ for $t\in\{0,\dots,T-1\}$, we obtain the following:
\begin{equation}
\label{eqn:unroll recursion}
z(\mathcal{I}^*)-z(\mathcal{I}_f^{T-1})\le(z(\mathcal{I}^*)-z(\mathcal{I}_f^0))\prod_{t=1}^{T-1}\Big(1-\frac{h_{j_t}(1-\epsilon)}{h(\mathcal{I}^*)}\Big).
\end{equation}
Furthermore, one can show that $\prod_{t=1}^{T-1}\big(1-\frac{h_{j_t}(1-\epsilon)}{h(\mathcal{I}^*)}\big)\le\big(1-\frac{h(\mathcal{I}_f^{T-1})(1-\epsilon)}{(T-1)h(\mathcal{I}^*)}\big)^{T-1}\le e^{-(1-\epsilon)\frac{h(\mathcal{I}_f^{T-1})}{h(\mathcal{I}^*)}}$ (e.g., \cite{khuller1999budgeted}). Since $z(\mathcal{I}_f^0)=z(\emptyset)=0$ and $z(\mathcal{I}^*)=z([n])$, it then follows from \eqref{eqn:unroll recursion} that
\begin{align}\nonumber
&z(\mathcal{I}^*)-z(\mathcal{I}_f^{T-1})\le z(\mathcal{I}^*)e^{-(1-\epsilon)\frac{h(\mathcal{I}_f^{T-1})}{h(\mathcal{I}^*)}},\\\nonumber
\implies&\ln(z([n]-z(\mathcal{I}_f^{T-1})))\le-(1-\epsilon)\frac{h(\mathcal{I}_f^{T-1})}{h(\mathcal{I}^*)}+\ln z([n]),\\
\implies&h(\mathcal{I}_f^{T-1})\le\frac{1}{1-\epsilon}\ln\frac{z([n])}{z([n])-z(\mathcal{I}_f^{T-1})}h(\mathcal{I}^*),\label{eqn:cost sum of T-1}
\end{align}
where we note that $z([n])-z(\mathcal{I}_f^{T-1})>0$, since $z(\cdot)$ is monotone nondecreasing and $z(\mathcal{I}_f^{T-1})\neq z([n])$. In order to prove part (a) (for $T\ge2$), it remains to show that $h_{j_T}\le\frac{1}{1-\epsilon} h(\mathcal{I}^*)$, which together with \eqref{eqn:cost sum of T-1} yield the bound in part (a). We can now use \eqref{eqn:mean lower bound} with $t=T-1$ to obtain
\begin{align}\nonumber
h_{j_T}&\le\frac{h(\mathcal{I}^*\setminus\mathcal{I}_f^{T-1})}{1-\epsilon}\cdot\frac{z(\mathcal{I}_f^T)-z(\mathcal{I}_f^{T-1})}{\sum_{j\in\mathcal{I}^*\setminus\mathcal{I}_f^{T-1}}(z(\mathcal{I}_f^{T-1}\cup\{j\})-z(\mathcal{I}_f^{T-1}))}\\
&\le\frac{h(\mathcal{I}^*)}{1-\epsilon}\cdot\frac{z(\mathcal{I}_f^T)-z(\mathcal{I}_f^{T-1})}{z(\mathcal{I}_f^{T-1}\cup\mathcal{I}^*)-z(\mathcal{I}_f^{T-1})},\label{eqn:cost of T}
\end{align}
where \eqref{eqn:cost of T} follows from the submodularity of $z(\cdot)$. Since $z(\mathcal{I}_f^{T})\le z(\mathcal{I}_f^{T-1}\cup\mathcal{I}^*)$ from the facts that $z(\mathcal{I}^*)=z([n])$ and $z(\cdot)$ is monotone nondecreasing, we see from \eqref{eqn:cost of T} that $h_{j_T}\le\frac{1}{1-\epsilon} h(\mathcal{I}^*)$. 

Next, suppose $T=1$, i.e., $\mathcal{I}_f=j_1$. We will show that $h(\mathcal{I}^*)=h(\mathcal{I}_g)$. Noting from the definition of Algorithm~$\ref{algorithm:fast greedy}$ that $j_1\in\mathop{\arg}{\max}_{i\in[n]}\frac{z(i)-z(\emptyset)}{h_i}$, we have
\begin{equation*}
\frac{z(j_1)}{h_{j_1}}\ge\frac{z(j)}{h_j},\ \forall j\in\mathcal{I}^*.
\end{equation*} 
It then follows from similar arguments to those for \eqref{eqn:mean lower bound} and \eqref{eqn:submodular lower bound} that 
\begin{equation*}
\frac{z(j_1)}{h_{j_1}}\ge\frac{\sum_{j\in\mathcal{I}^*}z(j)}{\sum_{j\in\mathcal{I}^*}h_j}\ge\frac{z(\mathcal{I}^*)}{h(\mathcal{I}^*)},
\end{equation*}
which implies
\begin{equation*}
\frac{h(\mathcal{I}_f)}{h(\mathcal{I}^*)}\le\frac{z(\mathcal{I}_f)}{z(\mathcal{I}^*)}\le1,
\end{equation*}
where we use the fact $z(\mathcal{I}_f)\le z(\mathcal{I}^*)$, since $z(\cdot)$ is monotone nondecreasing with $z(\mathcal{I}^*)=z([n])$. Thus, we have $h(\mathcal{I}_f)\le h(\mathcal{I}^*)$. Noting that $h(\mathcal{I}^*)\le h(\mathcal{I}_g)$ always holds due to the fact that $\mathcal{I}^*$ is an optimal solution, we conclude that $h(\mathcal{I}^*)=h(\mathcal{I}_g)$. This completes the proof of part (a).

Part (b) now follows directly from part (a) by noting that $z([n])-z(\mathcal{I}_f^{T-1})\ge1$, since $z([n])-z(\mathcal{I}_f^{T-1})>0$ and $z(\mathcal{I})\in\mathbb{Z}_{\ge1}$ for all $\mathcal{I}\subseteq[n]$.
\end{proof}

\begin{remark}
The threshold-based greedy algorithm has also been proposed for the problem of maximizing a monotone nondecreasing submodular function subject to a cardinality constraint (e.g., \cite{badanidiyuru2014fast}). The threshold-based greedy algorithm proposed in \cite{badanidiyuru2014fast} improves the running times of the standard greedy algorithm proposed in \cite{nemhauser1978analysis}, and achieves a comparable performance guarantee to that of the standard greedy algorithm in \cite{nemhauser1978analysis}. Here, we propose a threshold-based greedy algorithm (Algorithm~$\ref{algorithm:fast greedy}$) to solve the submodular set covering problem, which improves the running times of the standard greedy algorithm for the submodular set covering problem proposed in \cite{wolsey1982analysis} (i.e., Algorithm~$\ref{algorithm:fast greedy}$), and achieves comparable performances guarantees as we showed in Theorem~$\ref{thm:fast greedy}$.
\end{remark}

\subsection{Interpretation of Performance Bounds}\label{sec:illustrative examples}
Here, we give an illustrative example to interpret the performance bounds of Algorithm~$\ref{algorithm:greedy}$ and Algorithm~$\ref{algorithm:fast greedy}$ given in Theorem~$\ref{thm:guarantee for greedy SPL}$ and Theorem~$\ref{thm:fast greedy}$, respectively. In particular, we focus on the bounds given in Theorem~$\ref{thm:guarantee for greedy SPL}$(d) and Theorem~$\ref{thm:fast greedy}$(b). Consider an instance of the BLDS problem, where we set $\mu_0(\theta_p)=\frac{1}{m}$ for all $\theta_p\in\Theta$ with $m=|\Theta|$. In other words, there is a uniform prior belief over the states in $\Theta=\{\theta_1,\dots,\theta_m\}$. Moreover, we set the error bounds $R_{\theta_p}=\frac{R}{m}$ for all $\theta_p\in\Theta$, where $R\in\mathbb{Z}_{\ge0}$ and $R<m-1$. Recalling that $\bar{\Theta}=\{\theta_p\in\Theta: 0\le R_{\theta_p}<1-\mu_0(\theta_p)\}$ and noting the definition of $z(\cdot)$ in Eq.~\eqref{eqn:def of z}, for all $\mathcal{I}\subseteq[n]$, we define
\begin{equation}
\label{eqn:def of z prime}
z'(\mathcal{I})\triangleq m(m-R)z(\mathcal{I})=m(m-R)\sum_{\theta_p\in\Theta}f'_{\theta_p}(\mathcal{I}).
\end{equation}
One can check that $z'(\mathcal{I})\in\mathbb{Z}_{\ge0}$ for all $\mathcal{I}\subseteq[n]$. Moreover, one can show that \eqref{eqn:equivalent form of spl 1} can be equivalently written as 
\begin{equation}
\label{eqn:equivalent form of spl 2}
\begin{split}
&\mathop{\min}_{\mathcal{I}\subseteq[n]} h(\mathcal{I})\\
&s.t.\ z'(\mathcal{I})=z'([n]).
\end{split}
\end{equation}
Noting that $M'\triangleq\mathop{\max}_{j\in[n]}z'(j)\le m^2(m-R)$ from \eqref{eqn:def of z prime}, we then see from Theorem~$\ref{thm:guarantee for greedy SPL}$(d) that applying Algorithm~$\ref{algorithm:greedy}$ to \eqref{eqn:equivalent form of spl 2} yields the following performance bound:
\begin{equation}
\label{eqn:ex bound for greedy}
h(\mathcal{I}_g)\le\Big(\sum_{i=i}^{M'}\frac{1}{i}\Big)h(\mathcal{I}^*)\le\big(1+\ln M')h(\mathcal{I}^*)\le(1+2\ln m+\ln(m-R)\big)h(\mathcal{I}^*).
\end{equation}
Similarly, since $z'([n])\le m^2(m-R)$ also holds, Theorem~$\ref{thm:fast greedy}$(b) implies the following performance bound for Algorithm~$\ref{algorithm:fast greedy}$ when applied to \eqref{eqn:equivalent form of spl 2}:
\begin{equation}
\label{eqn:ex bound for fast greedy}
h(\mathcal{I}_f)\le\frac{1}{1-\epsilon}\big(1+\ln z'([n])\big)h(\mathcal{I}^*)\le\frac{1}{1-\epsilon}(1+2\ln m+\ln(m-R)\big)h(\mathcal{I}^*),
\end{equation}
where $\epsilon\in(0,1)$. Again, we note from Theorem~$\ref{thm:fast greedy}$ that a smaller value of $\epsilon$ yields a tighter performance bound for Algorithm~$\ref{algorithm:fast greedy}$ (according to \eqref{eqn:ex bound for fast greedy}) at the cost of slower running times. Thus, supposing $m$ and $\epsilon$ are fixed, we see from \eqref{eqn:ex bound for greedy} and \eqref{eqn:ex bound for fast greedy} that the performance bounds of Algorithm~$\ref{algorithm:greedy}$ and Algorithm~$\ref{algorithm:fast greedy}$ become tighter as $R$ increases, i.e., as the error bound $R_{\theta_p}$ increases. On the other hand, supposing $R$ and $\epsilon$ are fixed, we see from \eqref{eqn:ex bound for greedy} and \eqref{eqn:ex bound for fast greedy} that the performance bounds of Algorithm~$\ref{algorithm:greedy}$ and Algorithm~$\ref{algorithm:fast greedy}$ become tighter as $m$ decreases, i.e., as the number of possible states of the world decreases.

Finally, we note that the performance bounds given in Theorem~$\ref{thm:guarantee for greedy SPL}$ are {\it worst-case} performance bounds for Algorithm~$\ref{algorithm:greedy}$. Thus, in practice the ratio between a solution returned by the algorithm and an optimal solution can be {\it smaller} than the ratio predicted by Theorem~$\ref{thm:guarantee for greedy SPL}$. Nevertheless, there may also exist instances of the BLDS problem that let Algorithm~$\ref{algorithm:greedy}$ return a solution that meets the worst-case performance bound. Moreover, instances with tighter performance bounds (given by Theorem~$\ref{thm:guarantee for greedy SPL}$) potentially imply better performance of the algorithm when applied to those instances, as we can see from the above discussions and the numerical examples that will be provided in the next section. Therefore, the performance bounds given in Theorem~$\ref{thm:guarantee for greedy SPL}$ also provide insights into how different problem parameters of BLDS influence the actual performance of Algorithm~$\ref{algorithm:greedy}$. Similar arguments also hold for Algorithm~$\ref{algorithm:fast greedy}$ and the corresponding performance bounds given in Theorem~$\ref{thm:fast greedy}$.

\subsection{Numerical examples}
In this section, we focus on validating Algorithms~$\ref{algorithm:greedy}$ (resp., Algorithm $\ref{algorithm:fast greedy}$), and the performance bounds provided in Theorem~$\ref{thm:guarantee for greedy SPL}$ (resp., Theorem~$\ref{thm:fast greedy}$) using numerical examples constructed as follows. First, the total number of data sources is set to be $10$, and the selection cost $h_i$ is drawn uniformly from $[10]$ for all $i\in[n]$. The cost structure is then fixed in the sequel. Similarly to Section~$\ref{sec:illustrative examples}$, we consider BLDS instances where $\mu_0(\theta_p)=\frac{1}{m}$ for all $\theta_p\in\Theta$ with $m=|\Theta|$, and $R_{\theta_p}=\frac{R}{m}$ for all $\theta_p\in\Theta$ with $R\in\mathbb{Z}_{>0}$ and $R<m-1$. Specifically, we set $m=15$ and range $R$ from $0$ to $13$. For each $R\in\{0,1,\dots,13\}$, we further consider $500$ corresponding randomly generated instances of the BLDS problem, where for each BLDS instance we randomly generate the set $F_{\theta_p}^c(i)$ (i.e., the set of states that can be distinguished from $\theta_p$ given data source $i$) for all $i\in[n]$ and for all $\theta_p\in\Theta$.\footnote{Note that in the BLDS problem (Problem~$\ref{problem:SSL}$), the signal structure of each data source $i\in[n]$ is specified by the likelihood functions $\ell_i(\cdot|\theta_p)$ for all $\theta_p\in\Theta$. As we discussed in previous sections, \eqref{eqn:problem 1 obj} in Problem~$\ref{problem:SSL}$ can be equivalently written as \eqref{eqn:equivalent form of spl 1}, where one can further note that the function $z(\cdot)$ does not directly depend on any likelihood function $\ell_i(\cdot|\theta_p)$, and can be (fully) specified given $F_{\theta_p}^c(i)$ for all $i\in[n]$ and for all $\theta_p\in\Theta$. Thus, when constructing the BLDS instances in this section, we directly construct $F_{\theta_p}^c(i)$ for all $i\in[n]$ and for all $\theta_p\in\Theta$ in a random manner.} 

In Fig.~$\ref{fig:histograms}$ and Fig~$\ref{fig:performance bound}$, we showcase the results corresponding to Algorithm~$\ref{algorithm:greedy}$ when applied to solve the random BLDS instances generated above. Specifically, in Fig.~$\ref{fig:histograms}$, we plot histograms of the ratio $h(\mathcal{I}_g)/h(\mathcal{I}^*)$ for $R=1$, $R=5$ and $R=10$, where $\mathcal{I}_g$ is the solution returned by Algorithm~$\ref{algorithm:greedy}$ and $\mathcal{I}^*$ is an optimal solution to BLDS. We see from Fig.~$\ref{fig:histograms}$ that Algorithm~$\ref{algorithm:greedy}$ works well on the randomly generated BLDS instances, as the values of $h(\mathcal{I}_g)/h(\mathcal{I}^*)$ are close to $1$. Moreover, we see from Fig.~$\ref{fig:histograms}$ that as $R$ increases, Algorithm~$\ref{algorithm:greedy}$ yields better overall performance for the $500$ randomly generated BLDS instances. Now, from the way we set $\mu_0(\theta_p)$ and $R_{\theta_p}$ in the BLDS instances constructed above, we see from the arguments in Section~$\ref{sec:illustrative examples}$ that the performance bound for Algorithm~$\ref{algorithm:greedy}$ given by Theorem~$\ref{thm:guarantee for greedy SPL}$(d) can be written as $h(\mathcal{I}_g)\le\big(1+\ln M')h(\mathcal{I}^*)$, where $M'=\mathop{\max}_{j\in[n]}z'(j)$ and $z'(\cdot)$ is defined in \eqref{eqn:def of z prime}. Thus, in Fig.~$\ref{fig:performance bound}$, we plot the performance bound of Algorithm~$\ref{algorithm:greedy}$, i.e., $1+\ln M'$, for $R$ ranging from $0$ to $13$. Also note that for each $R\in\{0,1,\dots,13\}$, we obtain the averaged value of $1+\ln M'$ over $500$ random BLDS instances as we constructed above. We then see from Fig.~$\ref{fig:performance bound}$ that the value of the performance bound of Algorithm~$\ref{algorithm:greedy}$ decreases, i.e., the performance bound becomes tighter, as $R$ increases from $0$ to $13$. Since the performance bound in Theorem~$\ref{thm:guarantee for greedy SPL}$(d) is the worst-case guarantee, Algorithm~$\ref{algorithm:greedy}$ achieves better performance than that predicted by the bound. However, as we mentioned in Section~$\ref{sec:illustrative examples}$, the behavior of the performance bound aligns with the actual performance of Algorithm~$\ref{algorithm:greedy}$ presented in Fig.~$\ref{fig:histograms}$, i.e., a tighter performance bound implies a better overall performance of the algorithm on the $500$ random BLDS instances.

\begin{figure}[htbp]
\centering
\subfloat[a][$R=1$.]{
\includegraphics[width=0.3\linewidth]{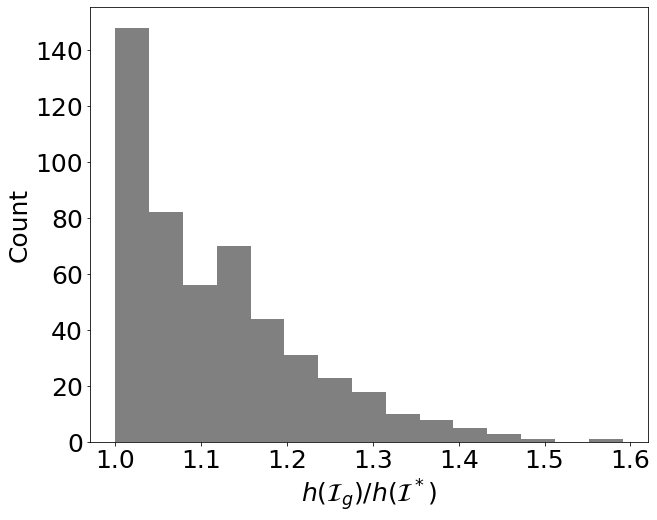}} % first figure itself
\subfloat[b][$R=5$.]{
\includegraphics[width=0.3\linewidth]{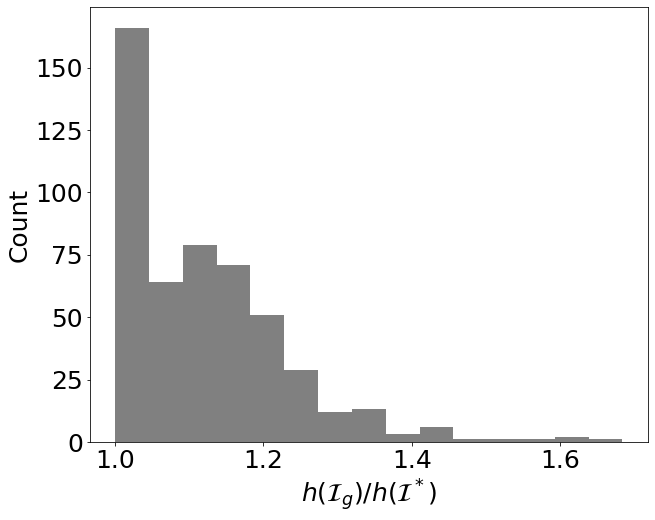}} % second figure itself
\subfloat[c][$R=10$.]{
\includegraphics[width=0.3\linewidth]{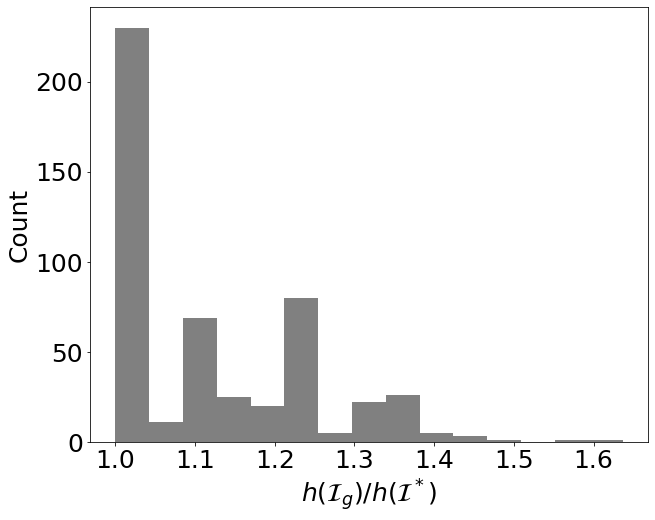}} % first figure itself
\caption{Histograms of the ratio $h(\mathcal{I}_g)/h(\mathcal{I}^*)$.}
\label{fig:histograms}
\end{figure}

\begin{figure}[htbp]
    \centering
    \includegraphics[width=0.3\linewidth]{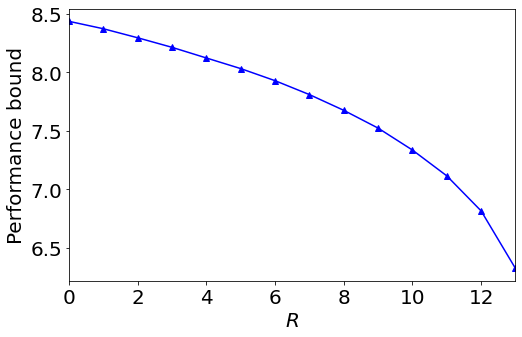} % first figure itself
    \caption{Performance bound for Algorithm~$\ref{algorithm:greedy}$ given by Theorem~$\ref{thm:guarantee for greedy SPL}$(d).}
\label{fig:performance bound}
\end{figure}

Similarly, we plot the results corresponding to Algorithm~$\ref{algorithm:fast greedy}$ when applied to the $500$ randomly generated BLDS instances as we described above. In addition, we set $\epsilon=0.1$ in Algorithm~$\ref{algorithm:fast greedy}$. Again, we observe from the histograms in Fig.~$\ref{fig:histograms f}$ that Algorithm~$\ref{algorithm:fast greedy}$ works well on the randomly generated BLDS instances, and that as $R$ increases, Algorithm~$\ref{algorithm:fast greedy}$ yields better over all performance for the $500$ randomly generated BLDS instances. Here, we also note from the histogram in Fig.~$\ref{fig:histograms f}$(b) that the ratio $h(\mathcal{I}_f)/h(\mathcal{I}^*)$ may be smaller than $1$ for {\it certain} BLDS instances (where recall that $h(\mathcal{I}_f)$ is the cost of the solution $\mathcal{I}_f$ returned by Algorithm~\ref{algorithm:fast greedy}). This is because the solution $\mathcal{I}_f$ returned by Algorithm~\ref{algorithm:fast greedy} only satisfies $z(\mathcal{I}_f)\ge(1-\epsilon) z([n])$ (where $z([n])=z(\mathcal{I}^*)$) as we argued in Theorem~$\ref{algorithm:fast greedy}$, which potentially implies that $z(\mathcal{I}_f)<z(\mathcal{I}^*)$ and $h(\mathcal{I}_f)/h(\mathcal{I}^*)<1$. Nonetheless, we observe from our experiments that for more than $99\%$ of the $1500$ random BLDS instances (with $R\in\{1,5,10\}$), the constraint $z(\mathcal{I}_f)=z([n])$ is satisfied. Moreover, we have from the arguments in Section~$\ref{sec:illustrative examples}$ that the performance bound for Algorithm~$\ref{algorithm:fast greedy}$ given by Theorem~$\ref{thm:fast greedy}$(b) can be written as $h(\mathcal{I}_f)\le\frac{1}{1-\epsilon}(1+\ln z'([n]))h(\mathcal{I}_*)$, where $z'(\cdot)$ is defined in \eqref{eqn:def of z prime} and we set $\epsilon=0.1$. In Fig.~\ref{fig:performance bound f}, we plot the performance bound of Algorithm~$\ref{algorithm:fast greedy}$, i.e., $\frac{1}{1-\epsilon}(1+\ln z'([n]))$, averaged over the $500$ random BLDS instances, for $R$ ranging from $0$ to $13$. We also see from Fig.~$\ref{fig:performance bound f}$ that the value of the performance bound of Algorithm~$\ref{algorithm:fast greedy}$ decreases, i.e., the performance bound becomes tighter, as $R$ increases from $0$ to $13$. Although the performance bound in Theorem~$\ref{thm:fast greedy}$ is still a worst-case guarantee, the behavior of the bound again aligns with the actual performance of Algorithm~$\ref{algorithm:fast greedy}$ presented in Fig.~$\ref{fig:histograms f}$, i.e., a tighter performance bound implies a better overall performance of the algorithm on the $500$ random BLDS instances. 

\begin{figure}[htbp]
\centering
\subfloat[a][$R=1$.]{
\includegraphics[width=0.3\linewidth]{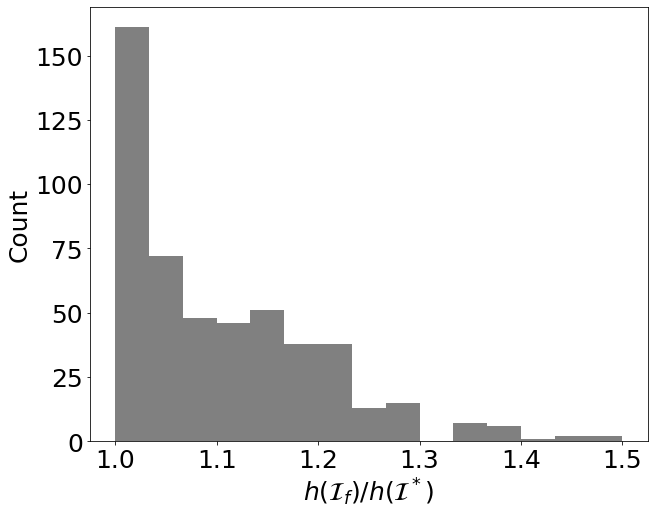}} % first figure itself
\subfloat[b][$R=10$.]{
\includegraphics[width=0.3\linewidth]{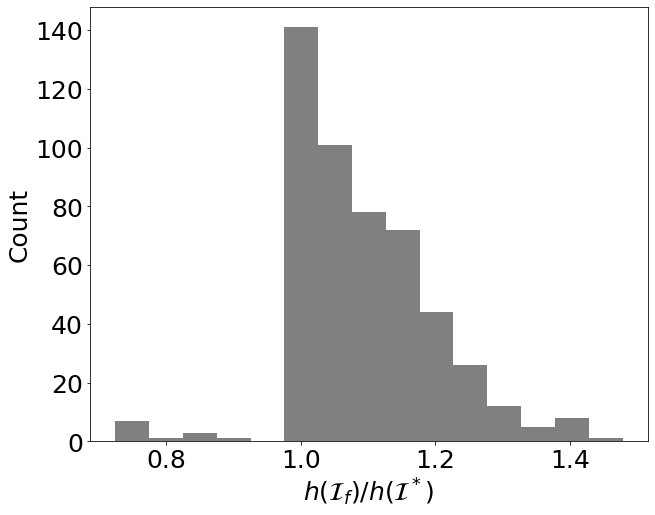}} % second figure itself
\subfloat[c][$R=10$.]{
\includegraphics[width=0.3\linewidth]{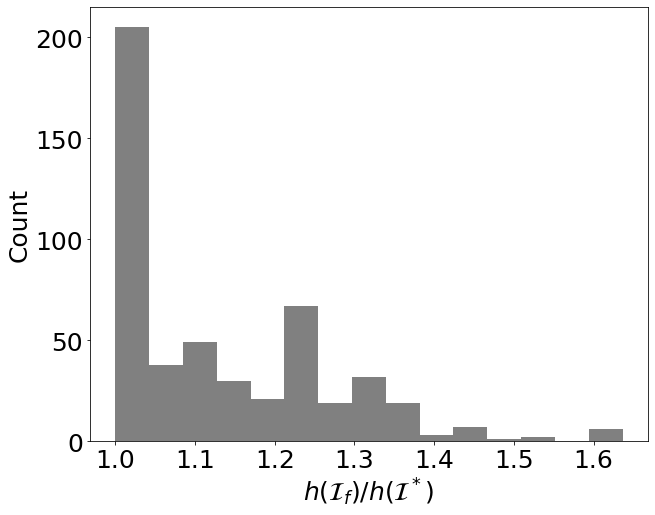}} % third figure itself
\caption{Histograms of the ratio $h(\mathcal{I}_f)/h(\mathcal{I}^*)$.}
\label{fig:histograms f}
\end{figure}

\begin{figure}[htbp]
    \centering
    \includegraphics[width=0.3\linewidth]{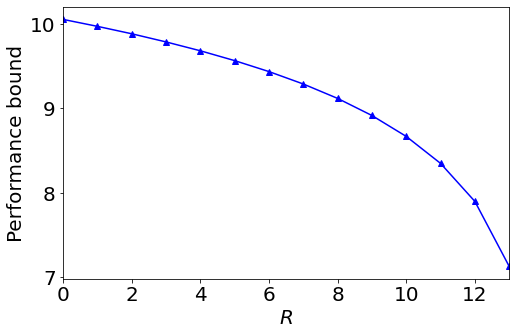} % first figure itself
    \caption{Performance bound for Algorithm~$\ref{algorithm:fast greedy}$ given by Theorem~$\ref{thm:fast greedy}$(b).}
\label{fig:performance bound f}
\end{figure}

Putting the above results and discussions together, both of Algorithms~$\ref{algorithm:greedy}$ and $\ref{algorithm:fast greedy}$ achieve good performance for the randomly generated BLDS instances, while Algorithm~$\ref{algorithm:fast greedy}$ achieves faster running times as we discussed in Section~$\ref{sec:fast greedy}$. Moreover, while the performance bound given in Theorem~$\ref{thm:guarantee for greedy SPL}$(d) for Algorithm~$\ref{algorithm:greedy}$ is tighter than that given in Theorem~$\ref{thm:fast greedy}$(b), both of the bounds provide insights into how the problem parameters of BLDS (e.g., the error bound $R$) influence the actual performance of the algorithms as we discussed above.

\section{Conclusion}
In this work, we considered the problem of data source selection for Bayesian learning. We first proved that the data source selection problem for Bayesian learning is NP-hard.  Next, we showed that the data source selection problem can be transformed into an instance of the submodular set covering problem, and can then be solved using a standard greedy algorithm with provable performance guarantees. We also proposed a fast greedy algorithm that improves the running times of the standard greedy algorithm, while achieving comparable performance guarantees. The fast greedy algorithm can be applied to solve the general submodular set covering problem. We showed that the performance bounds provide insights into the actual performances of the algorithms under different instances of the data source selection problem. Finally, we validated our theoretical analysis using numerical examples, and showed that the greedy algorithms work well in practice.

\bibliography{bibliography}

\section{Appendix}
\subsection{Extension to Non-Bayesian Learning}
\label{sec:problem formulation for nonBayesian}
Let us consider a scenario where there is a set of agents, denoted as $[n]$, who wish to {\it collaboratively} learn the true state of the world. The agents interact over a directed graph $\mathcal{G}=([n],\mathcal{E})$, where each vertex in $[n]$ corresponds to an agent and a directed edge $(j,i)\in\mathcal{E}$ indicates that agent $i$ can (directly) receive information from agent $j$. Denote $\mathcal{N}_i\triangleq\{j:(j,i)\in\mathcal{E}, j\neq i\}$ as the set of neighbors of agent $i$. Suppose each agent has an associated data source with the same observation model as described in Section \ref{sec:problem formulation for Bayesian}. Specifically, the observation (conditioned on the state $\theta\in\Theta$) provided by the data source at agent $i$ at time step $k\in\mathbb{Z}_{\ge1}$ is denoted as $\omega_{i,k}\in S_i$, which is generated by the likelihood function $\ell_i(\cdot|\theta)$. Each agent $i\in[n]$ is assumed to know $\ell_i(\cdot|\theta)$ for all $\theta\in\Theta$. Similarly, we consider the scenario where using the data source of agent $i\in[n]$ incurs a cost denoted as $h_i\in\mathbb{R}_{>0}$ for all $i\in[n]$, and there is a (central) designer who can select a subset $\mathcal{I}\subseteq[n]$ of agents whose data sources will be used to collect observations. We assume that the designer knows $\ell_i(\cdot|\theta)$ for all $i\in[n]$ and for all $\theta\in\Theta$. After set $\mathcal{I}\subseteq[n]$ is selected, each agent $i\in[n]$ updates its belief of the state of the world, denoted as $\mu^{\mathcal{I}}_{i,k}(\cdot)$, using the following distributed non-Bayesian learning rule as described in  \cite{nedic2017fast}:
\begin{equation}
\label{eqn:nonbayesian rule 1}
\mu_{i,k+1}^{\mathcal{I}}(\theta)=\frac{\prod_{j=1}^n(\mu_{j,k}^{\mathcal{I}}(\theta))^{a_{ij}}\ell_i(\omega_{i,k+1}|\theta)}{\sum_{\theta_p\in\Theta}\prod_{j=1}^n(\mu_{j,k}^{\mathcal{I}}(\theta_p))^{a_{ij}}\ell_i(\omega_{i,k+1}|\theta_p)}\ \forall\theta\in\Theta,
\end{equation}
where $\mu_{i,k}^{\mathcal{I}}(\theta)$ is the belief of agent $i$ that $\theta$ is the true state at time step $k$ when the set of sources given by $\mathcal{I}\subseteq[n]$ is selected, and $a_{ij}$ is the weight that agent $i\in[n]$ assigns to an agent $j\in\mathcal{N}_i\cup\{i\}$. Specifically, for any two distinct agents $i,j\in[n]$, $a_{ij}>0$ if agent $i$ receives information from agent $j$ and $a_{ij}=0$ otherwise, where $\sum_{j\in\mathcal{N}_i\cup\{i\}}a_{ij}=1$. Note that if agent $i\notin\mathcal{I}$, i.e., the data source of agent $i$ is not selected to collect observations, we set $\ell_i(s_i|\theta_p)=\ell_i(s_i|\theta_q)$ for all $\theta_p,\theta_q\in\Theta$ and for all $s_i\in S_i$. Similarly, for any $i\in[n]$, the initial belief is set to be $\mu_{i,0}^{\mathcal{I}}(\theta)=\mu_{i,0}(\theta)$ for all $\mathcal{I}\subseteq[n]$ and for all $\theta\in\Theta$, where $\sum_{\theta\in\Theta}\mu_{i,0}(\theta)=1$ and $\mu_{i,0}(\theta)\in\mathbb{R}_{\ge0}$ for all $\theta\in\Theta$. We then see from \eqref{eqn:nonbayesian rule 1} that $\sum_{\theta\in\Theta}\mu^{\mathcal{I}}_{i,k}(\theta)=1$ and $0\le\mu_{i,k}^{\mathcal{I}}(\theta)\le1$ for all $k\in\mathbb{Z}_{\ge0}$, for all $\theta\in\Theta$ and for all $\mathcal{I}\subseteq[n]$. Moreover, for a given true state $\theta\in\Theta$, we define $F_{\theta}(i)=\{\theta_p\in\Theta:\ \ell_i(s_i|\theta_p)=\ell_i(s_i|\theta),\forall s_i\in S_i\}$ for all $i\in[n]$. Similarly to Section~\ref{sec:problem formulation for Bayesian}, we denote $F_{\theta}(\mathcal{I})=\bigcap_{n_i\in\mathcal{I}} F_{\theta}(n_i)$, where we also assume that Assumption~$\ref{assump:likelihood function}$ holds for the analysis in this section. Again, note that $F_{\theta}(\emptyset)=\Theta$, and $\theta\in F_{\theta}(\mathcal{I})$ for all $\theta\in\Theta$ and for all $\mathcal{I}\subseteq [n]$.  We have the following result.
\begin{lemma}
\label{lemma:convergence of nonbayesian rule}
Consider a set $[n]$ of agents interacting over a strongly connected graph $\mathcal{G}=([n],\mathcal{E})$.\footnote{A directed graph $\mathcal{G}=([n],\mathcal{E})$ is said to be strongly connected if for each pair of distinct vertices $i,j\in[n]$, there exists a directed path (i.e., a sequence of directed edges) from $j$ to $i$.}  Suppose the true state of the world is $\theta^*$, $\mu_{i,0}(\theta)>0$ for all $i\in[n]$ and for all $\theta\in\Theta$, and $a_{ii}>0$ for all $i\in[n]$ in the rule given in \eqref{eqn:nonbayesian rule 1}. For any $\mathcal{I}\in [n]$, the rule given in \eqref{eqn:nonbayesian rule 1} ensures that (a) $\mathop{\lim}_{k\to\infty}\mu^{\mathcal{I}}_{i,k}(\theta_p)=0$ a.s. for all $\theta_p\notin F_{\theta^*}(\mathcal{I})$ and for all $i\in[n]$; and (b) $\lim_{k\to\infty}\mu^{\mathcal{I}}_{i,k}(\theta_q)=\frac{\prod_{j=1}^n\mu_{j,0}(\theta_q)^{\pi_j}}{\sum_{\theta\in F_{\theta^*}(\mathcal{I})}\prod_{j=1}^n\mu_{j,0}(\theta)^{\pi_j}}$ a.s. for all $i\in[n]$ and $\theta_q\in\ F_{\theta^*}(\mathcal{I})$, where $\pi\triangleq\begin{bmatrix}\pi_1 & \cdots & \pi_n\end{bmatrix}^{\prime}$ satisfies $\pi^{\prime}A=\pi^{\prime}$ and $\|\pi\|_1=1$, and $A\in\mathbb{R}^{n\times n}$ is defined such that $A_{ij}=a_{ij}$ for all $i,j\in[n]$. 
\end{lemma}
\begin{proof}
We begin by defining the following quantities for all $\mathcal{I}\subseteq[n]$, for all $i\in[n]$ and for all $k\in\mathbb{Z}_{\ge0}$:
\begin{equation}
\label{eqn:def of two tool ratios}
\delta^{\mathcal{I}}_{i,k}(\theta)\triangleq \ln\frac{\mu^{\mathcal{I}}_{i,k}(\theta)}{\mu^{\mathcal{I}}_{i,k}(\theta^*)}\ \text{and}\ \sigma_{i,k+1}(\theta) \triangleq \ln\frac{\ell_i(\omega_{i,k+1}|\theta)}{\ell_i(\omega_{i,k+1}|\theta^*)},
\end{equation}
where $\delta_{i,0}^{\mathcal{I}}(\theta)=\delta_{i,0}(\theta)\triangleq\ln\frac{\mu_{i,0}(\theta)}{\mu_{i,0}(\theta^*)}$ for all $\mathcal{I}\subseteq[n]$. For any  $\mathcal{I}\subseteq[n]$, we consider an agent $i\in[n]$ and $\theta_p\notin F_{\theta^*}(\mathcal{I})$. Following similar arguments to those for Theorem~$1$ in \cite{nedic2017fast}, one can obtain that $\lim_{k\to\infty}\delta_{i,k}^{\mathcal{I}}(\theta_p)=-\infty$ a.s., i.e., $\lim_{k\to\infty}\frac{\mu^{\mathcal{I}}_{i,k}(\theta_p)}{\mu^{\mathcal{I}}_{i,k}(\theta^*)}=0$ a.s. Since $0\le\mu^{\mathcal{I}}_{i,k}(\theta)\le1$ for all $\theta\in\Theta$ and for all $k\in\mathbb{Z}_{\ge0}$, it follows that $\lim_{k\to\infty}\frac{\mu^{\mathcal{I}}_{i,k}(\theta_p)}{\mu^{\mathcal{I}}_{i,k}(\theta^*)}\ge\lim_{k\to\infty}\mu^{\mathcal{I}}_{i,k}(\theta_p)\ge0$, which implies $0\le\lim_{k\to\infty}\mu^{\mathcal{I}}_{i,k}(\theta_p)\le0$ a.s., i.e., $\lim_{k\to\infty}\mu^{\mathcal{I}}_{i,k}(\theta_p)=0$ a.s. This proves part (a). 

We then prove part (b). For any  $\mathcal{I}\subseteq[n]$, we now consider an agent $i\in[n]$ and $\theta_q\in F_{\theta^*}(\mathcal{I})$. Based on the definition of $F_{\theta^*}(\mathcal{I})$, we note that $\sigma_{i,k+1}(\theta_q)=0, \forall k\in\mathbb{Z}_{\ge0}$. We then obtain from \eqref{eqn:nonbayesian rule 1} the following:
\begin{equation*}
\delta^{\mathcal{I}}_{k+1}(\theta_q)=A\delta^{\mathcal{I}}_k(\theta_q),
\end{equation*}
where $\delta^{\mathcal{I}}_{k}(\theta_q)\triangleq\begin{bmatrix}\delta^{\mathcal{I}}_{1,k}(\theta_q) & \cdots & \delta^{\mathcal{I}}_{n,k}(\theta_q)\end{bmatrix}^{\prime}$. Moreover, we have
\begin{equation} 
\label{eqn:limiting_dist}
\lim_{k\to\infty}\delta^{\mathcal{I}}_{k}(\theta_q)=(\lim_{k\to\infty} A^k ) \delta_{0}(\theta_q)=\mathbf{1}_n\pi^{\prime}\delta_{0}(\theta_q),
\end{equation}
where the last equality follows from the fact that $A$ is an irreducible and aperiodic stochastic matrix based on the hypotheses of the lemma. Simplifying \eqref{eqn:limiting_dist}, we obtain
\begin{equation}
\label{eqn:balance_eqn}
\lim_{k\to\infty}\frac{\mu^{\mathcal{I}}_{i,k}(\theta_q)}{\mu^{\mathcal{I}}_{i,k}(\theta^*)}=\frac{\prod_{j=1}^n\mu_{j,0}(\theta_q)^{\pi_j}}{\prod_{j=1}^n\mu_{j,0}(\theta^*)^{\pi_j}}>0.
\end{equation}
Summing up Eq.~\eqref{eqn:balance_eqn} for all $\theta_q\in F_{\theta^*}(\mathcal{I})$, we have
\begin{equation}
\label{eqn:balance_eqn sum}
\lim_{k\to\infty}\frac{\sum_{\theta_q\in F_{\theta^*}(\mathcal{I})}\mu^{\mathcal{I}}_{i,k}(\theta_q)}{\mu^{\mathcal{I}}_{i,k}(\theta^*)}=\sum_{\theta_q\in F_{\theta^*}(\mathcal{I})}\frac{\prod_{j=1}^n\mu_{j,0}(\theta_q)^{\pi_j}}{\prod_{j=1}^n\mu_{j,0}(\theta^*)^{\pi_j}}>0.
\end{equation}
Noting from part (a) that $\lim_{k\to\infty}\sum_{\theta_q\in\ F_{\theta^*}(\mathcal{I})}\mu^{\mathcal{I}}_{i,k}(\theta_q)=1$ a.s., we see from \eqref{eqn:balance_eqn sum} that $\lim_{k\to\infty}\mu^{\mathcal{I}}_{i,k}(\theta^*)$ exists and is positive, a.s., which further implies via \eqref{eqn:balance_eqn} that $\lim_{k\to\infty}\mu^{\mathcal{I}}_{i,k}(\theta_q)$ exists and is positive, a.s. In other words, we have from  \eqref{eqn:balance_eqn} the following:
\begin{equation}
\label{eqn:balance_eqn_limit}
\frac{\mu^{\mathcal{I}}_{i,\infty}(\theta_q)}{\mu^{\mathcal{I}}_{i,\infty}(\theta^*)}=\frac{\prod_{j=1}^n\mu_{j,0}(\theta_q)^{\pi_j}}{\prod_{j=1}^n\mu_{j,0}(\theta^*)^{\pi_j}},
\end{equation}
where $\mu^{\mathcal{I}}_{i,\infty}(\theta_q)\triangleq\lim_{k\to\infty}\mu^{\mathcal{I}}_{i,k}(\theta_q)$ for all $\theta_q\in F_{\theta^*}(\mathcal{I})$. Again noting that $\lim_{k\to\infty}\sum_{\theta\in\ F_{\theta^*}(\mathcal{I})}\mu^{\mathcal{I}}_{i,k}(\theta)=1$ a.s. for all $i\in[n]$, part (b) then follows from Eq.~\eqref{eqn:balance_eqn_limit}.
\end{proof}

Similarly to the problem formulation described in Section \ref{sec:problem formulation for Bayesian}, we define the following error metric for the designer:
\begin{equation*}
\bar{e}_{\theta^*}(\mathcal{I})=\sum_{i=1}^n e_{\theta^*,i}(\mathcal{I}),
\end{equation*}
where $\theta^*$ is the true state, $e_{\theta^*,i}(\mathcal{I})\triangleq\frac{1}{2}\mathop{\lim}_{k\to\infty} \|\mu^{\mathcal{I}}_{i,k} - \mathbf{1}_{\theta^*}\|_1$ and $\mu^{\mathcal{I}}_{i,k}\triangleq\begin{bmatrix}\mu^{\mathcal{I}}_{i,k}(\theta_1) & \cdots & \mu^{\mathcal{I}}_{i,k}(\theta_m)\end{bmatrix}^{\prime}$. In words, $\bar{e}_{\theta^*}(\mathcal{I})$ is the sum of the steady-state learning errors of all the agents in $[n]$, when the true state of the world is assumed to be $\theta^*$. It then follows from Lemma $\ref{lemma:convergence of nonbayesian rule}$ that $\bar{e}_{\theta^*}(\mathcal{I})=n\big(1-\frac{\prod_{j=1}^n(\mu_{j,0}(\theta^*))^{\pi_j}}{\sum_{\theta\in F_{\theta^*}(\mathcal{I})}\prod_{j=1}^n(\mu_{j,0}(\theta))^{\pi_j}}\big)$ almost surely. Denoting 
\begin{equation}
\label{eqn:error of nonbayesian case}
\bar{e}^s_{\theta_p}(\mathcal{I})\triangleq n\big(1-\frac{\prod_{j=1}^n(\mu_{j,0}(\theta_p))^{\pi_j}}{\sum_{\theta\in F_{\theta_p}(\mathcal{I})}\prod_{j=1}^n(\mu_{i,0}(\theta))^{\pi_j}}\big)\ \forall\theta_p\in\Theta,
\end{equation}
we consider the following problem for the designer:
\begin{equation}
\label{eqn:nonbayesian opt}
\begin{split}
&\mathop{\min}_{\mathcal{I}\subseteq [n]} h_{\mathcal{I}}\\
& s.t. \ \bar{e}^s_{\theta_p}(\mathcal{I})\le \bar{R}_{\theta_p},\forall\theta_p\in\Theta,
\end{split}
\end{equation}
where $0\le \bar{R}_{\theta_p}\le n$ and $\bar{R}_{\theta_p}\in\mathbb{R}$. Denoting $\bar{\mu}_0(\theta)\triangleq\prod_{i=1}^n\mu_{i,0}(\theta)^{\pi_i}$ for all $\theta\in\Theta$, we have from \eqref{eqn:error of nonbayesian case}: 
\begin{equation}
\label{eqn:error of nonbayesian case 2}
\bar{e}_{\theta_p}^s(\mathcal{I})=n(1-\frac{\bar{\mu}_0(\theta_p)}{\sum_{\theta\in F_{\theta_p}(\mathcal{I})}\bar{\mu}_0(\theta)}),\forall \theta_p\in\Theta.
\end{equation}
Now, we have from \eqref{eqn:error of bayesian case} and \eqref{eqn:error of nonbayesian case 2} that the optimization problem \eqref{eqn:nonbayesian opt} can be viewed as an instance of Problem $\ref{problem:SSL}$. Thus, all the theoretical results derived in this paper apply to this non-Bayesian distributed setting as well.

\end{document}